%% file: iclr2027_conference.tex
\PassOptionsToPackage{table}{xcolor}

\documentclass{article}

\usepackage{iclr2027_conference,times}

\input{math_commands.tex}

\usepackage{amsmath}
\usepackage{amsfonts}
\usepackage{amsthm}
\usepackage{nicefrac}
\usepackage{graphicx}
\usepackage{hyperref}
\usepackage{graphicx}
\usepackage{wrapfig}
\usepackage{subcaption}
\usepackage{booktabs}
\usepackage{multirow}
\usepackage{wrapfig}
\usepackage{xcolor}
\usepackage[most]{tcolorbox}
\usepackage{enumitem}

\usepackage{microtype}

\usepackage{url}
\usepackage{hyperref}

\definecolor{promptbrown}{RGB}{150,70,20}
\definecolor{promptgreen}{RGB}{0,110,45}
\definecolor{promptpurple}{RGB}{120,45,170}
\definecolor{promptblue}{RGB}{20,40,150}
\definecolor{promptgray}{RGB}{247,247,247}
\definecolor{codegray}{RGB}{245,245,245}

\newtheorem{theorem}{Theorem}
\newtheorem{lemma}{Lemma}

\newtheorem{assumption}{Assumption}

\title{
RIVO: Lightweight Adaptive Rerouting for Deep Research Agents
}

\author{Mingchen Li\textsuperscript{\normalfont 1}, Jiatan Huang\textsuperscript{\normalfont 2}, \textbf{Hong Yu}\textsuperscript{1}\\
\textsuperscript{1}University of Massachusetts, Amherst 
         \textsuperscript{2} University of Connecticut \\
        }

\iclrfinalcopy

\begin{document}

\maketitle

% ---------------------------------------------------------------------
% Main paper
% ---------------------------------------------------------------------

\input{1_abs.tex}
\input{2_introduction.tex}
\input{2_related_work.tex}
\input{3_preliminaries.tex}
\input{3_method.tex}
\input{4_experiment.tex}
\input{6_conclution.tex}

\bibliography{iclr2027_conference}
\bibliographystyle{iclr2027_conference}

% =====================================================================
% Appendix
% =====================================================================

\clearpage
\appendix
% =========================
% Appendix Navigation
% =========================
\input{9_appendix_navigation.tex}
\input{8_appendix.tex}
\input{7_dataset_build.tex}
\input{10_addtinial_relatedwork.tex}
\input{7_proof.tex}

\end{document}

%% file: math_commands.tex
\usepackage{amsmath,amsfonts,bm}

\def\eqref#1{equation~\ref{#1}}
\def\1{\bm{1}}

\DeclareMathAlphabet{\mathsfit}{\encodingdefault}{\sfdefault}{m}{sl}
\SetMathAlphabet{\mathsfit}{bold}{\encodingdefault}{\sfdefault}{bx}{n}

%% file: 1_abs.tex
% !TeX root = iclr2027_conference.tex

\begin{abstract}
Deep Research agents rely on long-horizon reasoning and iterative tool use to solve complex information-seeking tasks. However, adaptive decision-making over these trajectories remains challenging, as agents may stop prematurely when further evidence is still needed, continue unnecessarily after sufficient evidence has been gathered, or fail to redirect unproductive trajectories. Existing solutions either require costly model retraining and dense turn-level supervision, or rely solely on heuristic control rules that are not directly optimized for downstream research utility.
This motivates a lightweight, utility-aware alternative for deep research control. We formulate adaptive rerouting in long-horizon Deep Research as a state-conditioned trajectory correction problem. Based on this view, we introduce \textsc{RIVO}, a lightweight plug-in framework for adaptive trajectory correction.
The agent first estimates its research state from lightweight state signals and determines whether trajectory redirection is needed. 
Once triggered, the plug-in modifies the Host token logits to steer generation
toward a more promising research direction. Theoretical analysis shows that,
under a value-alignment condition and bounded local curvature, a sufficiently
small logit intervention can locally improve the expected downstream research
utility. A curated  dataset, BrowseCore, with a paired local corpus is further constructed
for efficient and controlled Deep Research experimentation. Building on this resource,
the plug-in is trained from downstream outcomes, enabling it to learn corrections that
improve subsequent research and final task performance.
Extensive experiments across multiple Deep Research benchmarks and control settings demonstrate the effectiveness of our method, showing that lightweight logit intervention can improve long-horizon research without retraining the underlying Host model
\footnote{%
The BrowseCore dataset is publicly available at
\href{https://huggingface.co/datasets/ChenAI2015/BrowseCore}{%
\raisebox{-0.15em}{\includegraphics[height=1em]{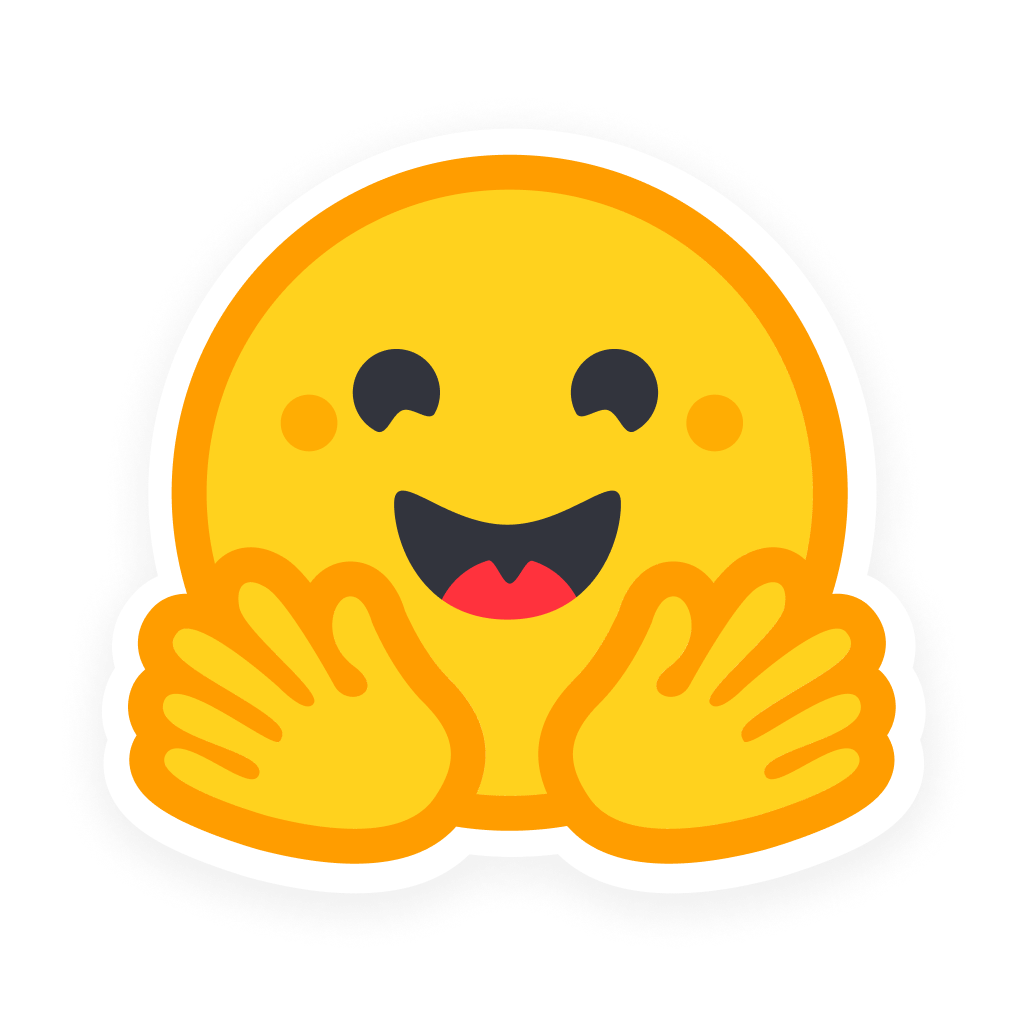}}%
}, and the code will be released on
\href{https://github.com/ToneLi/RIVO_deepresearch}{GitHub}.%
}.

\end{abstract}

% \footnote{The code is available at \url{https://anonymous.4open.science/r/RICE_PO-4F06/README.md}}.

%% file: 2_introduction.tex
% !TeX root = iclr2027_conference.tex
\section{Introduction}
Large language models (LLMs) have evolved from passive text generators into autonomous agents that interleave reasoning with external tool use, such as web search and knowledge retrieval \citep{alzubi2025open,li2026context,li2026efficient,li2026rice}. Building on this capability, \emph{Deep Research} agents perform long-horizon, iterative retrieval and reasoning to solve complex information-seeking tasks that require multi-hop evidence acquisition and synthesis\citep{dong2026s1,li2026literesearcher}. These agents are increasingly useful for applications such as web-scale question answering and patient diagnosis and clinical decision support \citep{wang2026deeper}. However, current Deep Research systems still suffer from poor search control, as agents may continue searching after sufficient evidence has been collected, stop prematurely when further evidence is still needed, or remain on an unproductive research path instead of redirecting toward more useful evidence \citep{soudani2026when}. This highlights a key limitation of current Deep Research agents.
Long-horizon research requires reliable decisions on when to stop, when to reroute,
and how to redirect the search toward more promising directions.

A natural way to improve research control is to further train the underlying Deep Research model. Supervised fine-tuning can teach explicit intervention behaviors, but it requires high-quality trajectories that specify not only when the agent should stop or reroute, but also how the research trajectory should be redirected. Collecting such supervision becomes increasingly difficult as research trajectories extend to dozens or even hundreds of interaction steps \citep{ICLR2024_29e8437d}. Reinforcement learning offers a more flexible alternative by optimizing search behavior from task rewards, yet long-horizon optimization introduces sparse and delayed feedback, difficult credit assignment, and substantial training cost due to repeated reasoning and retrieval rollouts \citep{wang2026information}. These costs become particularly significant when each step invokes large local retrievers or external APIs, and the resulting policy is typically coupled to a specific research backbone, requiring re-optimization when the underlying system changes. Training-free approaches avoid such retraining by relying on prompting or predefined internal and retrieval signals \citep{yao2023react,li2026stop}. However, these heuristic or proxy signals are not directly optimized for downstream research reward, and therefore cannot explicitly learn the long-term utility of an intervention,  whether the agent should stop, whether rerouting is beneficial, and, critically, which alternative search direction should be pursued.  These limitations naturally raise the following question:
\emph{Can we learn a lightweight, plug-and-play research controller that improves decision-making in long-horizon Deep Research without retraining the underlying agent or requiring substantial computational resources?}

Motivated by recent advances in inference-time logit reweighting for frozen language models \citep{pmlr-v267-hiranandani25a}, we formulate adaptive control in long-horizon Deep Research as a \emph{state-conditioned trajectory correction} problem. Based on this formulation, we introduce Reward-Learned Logit Intervention for Value Optimization (\textsc{RIVO}), a lightweight plug-in framework for adaptive trajectory correction. Rather than retraining the underlying research model, \textsc{RIVO} estimates the current research state from lightweight state signals to determine whether intervention is needed. When trajectory redirection is triggered, the plug-in applies a token-level correction to the frozen model logits, steering generation toward a more promising research direction. \textsc{RIVO} is trained from downstream trajectory outcomes, allowing it to learn interventions that improve subsequent research and final task performance.

The key component of \textsc{RIVO}  is a sparse logit intervention for trajectory correction.
A lightweight plug-in generates a low-rank residual correction that is sparsely
applied to the Top-$K$ selected logit dimensions, while the Host model, i.e.,
the underlying Deep Research model, remains frozen. This
intervention directly changes selected token preferences and redirects subsequent
retrieval without end-to-end retraining. To support model development and controlled evaluation, we construct
BrowseCore, a Deep Research dataset with a paired local corpus. Across the evaluated Deep Research benchmarks, RIVO
outperforms strong baselines by up to 4.36 points. Further ablation studies validate the key components of RIVO.
For example, on BrowseCore, removing the reroute predictor reduces the Success Rate
by 5.67 points, while removing the logit intervention leads to a 2.48-point drop. Finally, theoretical analysis studies the effect of logit
intervention and shows that sufficiently small, value-aligned interventions can
locally improve expected downstream research utility.

\begin{itemize}

    \item We construct a curated question answer dataset paired with a local corpus to
support efficient and controlled training and evaluation of long-horizon Deep Research agents.

 \item We formulate adaptive rerouting in long-horizon Deep Research as a
\emph{state-conditioned trajectory correction} problem and introduce
\textsc{RIVO}, a lightweight plug-in framework for targeted trajectory redirection.
Theoretical analysis shows that sufficiently small, value-aligned logit interventions
can locally improve expected downstream research utility.
    \item We conduct extensive experiments across multiple Deep Research benchmarks and control settings. The results demonstrate the effectiveness of \textsc{RIVO} in improving research performance while avoiding end-to-end retraining of the underlying research model.
\end{itemize}

%% file: 2_related_work.tex
% !TeX root = iclr2027_conference.tex
\section{Related Work}
\paragraph{Deep Research Agents and Research Control.}
Recent Deep Research agents, including LiteResearcher-4B~\citep{li2026literesearcher},
Marco-DeepResearch-8B~\citep{zhu2026marco}, and
Tongyi-DeepResearch-30B-A3B~\citep{tongyidr},
develop long-horizon research capabilities through supervised and reward-based
post-training. Despite strong overall performance, these agents still exhibit
limitations in intermediate research control: they may terminate at suboptimal
points, continue after meaningful progress has stalled, or fail to redirect
unproductive trajectories~\citep{soudani2026when}. Recent work has begun to
study such decisions. CaRT~\citep{DBLP:journals/corr/abs-2510-08517} learns
when information gathering should terminate from counterfactual trajectories,
while LSP~\citep{liu-wang-2025-answer} studies adaptive early stopping based on
reasoning convergence. However, these methods are not designed for long-horizon
Deep Research. Extending such control requires either fine-grained supervision
over long trajectories or repeated long-horizon optimization, motivating
lightweight control without retraining the full research model.

\paragraph{Plug-in Adaptation for LLM Agents.}
Several approaches improve agent behavior through inference-time control.
ReAct~\citep{yao2023react} interleaves reasoning and environment actions through
prompting, while ASAG~\citep{li2026stop} uses attention states to adaptively
terminate unnecessary reasoning. RAAC~\citep{soudani2026when} further introduces
retrieval-aware signals to detect stagnation and adapt long-horizon Deep Research
trajectories. In contrast, PPDPP~\citep{deng2024plug} learns a lightweight
plug-in to guide the policy of a frozen LLM, but targets proactive dialogue
rather than long-horizon research. A complementary line of work directly
intervenes in model logits: ThinkMerge~\citep{wang2026think} combines parallel
reasoning trajectories through logit averaging, while SWAI~\citep{an2026steering}
steers generation through training-free logit modification. Our method combines
these directions by learning a lightweight plug-in that modifies frozen model
logits only when trajectory rerouting is needed.

\paragraph{Training Resources for Deep Research.}
Existing Deep Research benchmarks provide limited resources for training
long-horizon research control. BrowseComp-Plus~\citep{chen2025browsecomp}
and DRBench~\citep{abaskohi2026drbench} are primarily designed for evaluation
and do not provide dedicated training sets for learning trajectory-level
interventions. OpenResearcher~\citep{li2026openresearcher} instead releases
96K synthetically generated long-horizon research trajectories together with
an approximately 11B-token retrieval corpus. While this provides a valuable
training resource, its large-scale retrieval environment can be computationally
expensive for repeated policy optimization, particularly when coupled with a
large dense retriever. To support efficient and controlled training of research
intervention, we construct a compact training resource that pairs high-quality
question--answer examples with a dedicated local corpus, enabling repeated
long-horizon rollouts without requiring large-scale retrieval infrastructure.

%% file: 3_preliminaries.tex
\section{Problem Setup}

We consider a long-horizon Deep Research setting in which an LLM agent
interacts with a retrieval environment over multiple research
turns. Given an initial user question $q_0$, the agent maintains an
interaction history $\mathcal{F}_{t-1}$ containing previous reasoning,
tool calls, retrieved evidence, and tool observations. At each turn
$t=1,2,\ldots,T$, the frozen Host observes $q_0$ and
$\mathcal{F}_{t-1}$ and produces either a tool action $a_t$ or a final
answer.
A tool action is executed by the environment and returns an observation
$o_t = \mathcal{E}(a_t)$, which is appended to the interaction history. 
A complete research episode is represented as
$
\tau
=
\{(a_1,o_1),(a_2,o_2),\ldots,(a_T,o_T)\},
$
followed by a final answer. The environment assigns a task-level reward
$R(\tau)$ according to the quality of the final research outcome. Among trajectory-level control decisions, in  this work, we mainly study adaptive rerouting. At such a
state, the agent generates a new retrieval direction $q_t$, which changes
the subsequent research trajectory.
The quality of a rerouting decision is determined by its downstream effect
rather than by the reroute query alone. We quantify the effect of a reroute query $q_t$ on subsequent research as
\begin{equation}
Q_t(q_t)
=
\mathbb{E}
\left[
R(\tau_{1:T})
\mid
q_t,\mathcal{F}_{t-1}
\right],
\end{equation}
where $Q_t(q_t)$ represents the expected research utility after selecting
$q_t$ at turn $t$. Our objective is to learn a lightweight trajectory
correction mechanism that increases this utility while keeping the host
model frozen.

%% file: 3_method.tex
% !TeX root = iclr2027_conference.tex
\section{RIVO}
\label{sec:method}

\subsection{The Challenge of Efficient Research Control}

Improving long-horizon research control through supervised fine-tuning or
reinforcement learning can be computationally expensive. Supervised approaches
require fine-grained intervention trajectories, while reward-based optimization
depends on repeated long-horizon reasoning and retrieval rollouts with sparse and
delayed feedback. Training-free methods avoid model retraining, but typically rely
on predefined heuristics or proxy signals that are not directly optimized for
downstream research utility.
A lightweight plug-in offers a more efficient alternative, but requiring a small
model to independently generate complex research directions places substantial
generative burden on the controller. We instead retain the frozen Host as the
primary generator and use the plug-in to apply residual corrections in logit space.
State-guided control determines when intervention is necessary, while the plug-in
learns how to modify the Host generation from downstream trajectory-level outcomes.

    Motivated by this view, we propose \textsc{RIVO}, a lightweight framework for state-conditioned trajectory correction. As illustrated in Figure~\ref{fig:method}, \textsc{RIVO} follows a simple pipeline with three components. First, State-Guided Control monitors the evolving research state using lightweight signals and determines whether the current trajectory should stop, continue, or be redirected. Second, when rerouting is triggered, Logit Intervention activates a lightweight plug-in that applies a sparse residual correction to the frozen Host logits, steering generation toward a new retrieval direction without updating the Host parameters. Third, Optimization and Inference trains the plug-in from downstream trajectory-level feedback and selectively applies the learned intervention during inference. Together, these components allow \textsc{RIVO} to adaptively correct unproductive research trajectories while keeping the Host model frozen and limiting intervention to states where rerouting is needed. 
% \subsection{Overall Framework}
\label{sec:overview}

\begin{figure*}[t]
    \centering
    \includegraphics[
        width=\textwidth,
    ]{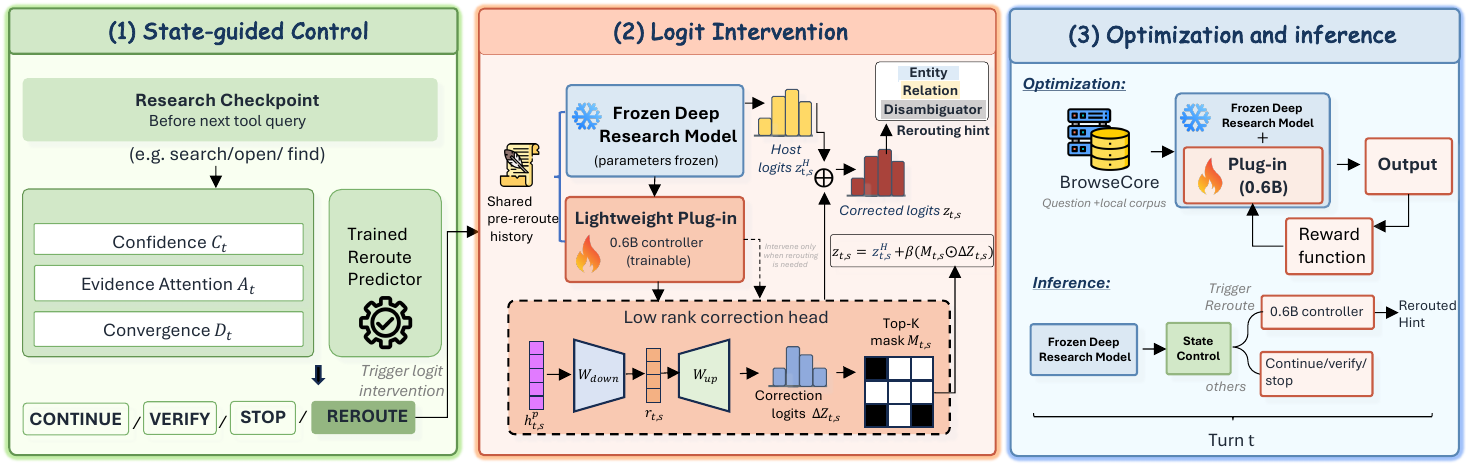}
    \caption{Overview of \textbf{RIVO}.
RIVO consists of three stages: state-guided control determines whether the current Deep Research trajectory should continue, stop, verify or reroute; logit intervention uses a lightweight plug-in to modify selected Host logits when rerouting is triggered; and trajectory-level optimization trains the plug-in while keeping the Host model frozen.
At inference time, the learned intervention is activated only when rerouting is required.}
    \label{fig:method}
    \vspace{-10pt}
\end{figure*}

\subsection{State-Guided Control}
\label{sec:initialization}

Figure~\ref{fig:method} summarizes the state-guided control process in
RIVO. At each retrieval checkpoint, RIVO evaluates the current research state
using three complementary signals: confidence $C_t$, convergence $D_t$, and
evidence attention $A_t$. Confidence identifies states that are sufficiently
reliable to stop, convergence determines whether low-confidence states require
further verification, and evidence attention distinguishes between continuing
the current trajectory and redirecting it. When the state-guided controller
selects \textsc{Continue}, a Trained Reroute Predictor performs an
additional learned check and may promote the decision to \textsc{Reroute}.
Once rerouting is triggered by either component, the plug-in applies token-level
logit correction to the frozen Host to redirect subsequent retrieval. 

% Detailed decision rules and triggering conditions are provided in Appendix~\ref{app:rivo}.
% % \subsubsection{Trajectory States and Three-Way Control Labels}
% \label{sec:control_labels}
\paragraph{Trained Reroute Predictor.}
We collect successful multi-turn Deep Research trajectories and extract a
turn-level state $s_t=(q,h_{<t},a_t)$, where $q$ is the original question,
$h_{<t}$ is the preceding research history, and $a_t$ is the current tool
action together with its result. A capable language model annotates each state
as \textsc{Continue} or \textsc{Reroute} based on the
accumulated evidence, current research direction, and gold answer.
This process yields 29,252 turn-level examples from 1,475 questions. We split
the data by question ID into 26,351 training examples and 2,901 test examples,
ensuring that all turns from the same question remain in a single partition.
We use the training split to fine-tune  Qwen3-0.6B-Base as a
lightweight reroute predictor, denoted by $P$.  An action-prediction head produces a distribution
$p_P(\cdot\mid s_t)$ over
$\{\textsc{Continue},\textsc{Reroute}\}$, where
$p_P(\textsc{Reroute}\mid s_t)$ denotes the predicted rerouting probability.
The same plug-in backbone is later used with a separate low-rank residual
correction head for logit intervention. Detailed annotation rules, prompts,
and data statistics are provided in Appendix~\ref{trajactory_label}.

\paragraph{Research-state estimation.}
At each retrieval checkpoint, RIVO estimates the current research state using
three lightweight signals. Provisional-answer confidence $C_t$ measures the
Host's confidence in its current answer using a fixed answer probe.
Evidence attention compares the attention assigned to the current and preceding
retrieval segments, while attention convergence measures the change in attention
entropy relative to the initial retrieval state. We define
\begin{equation}
C_t=\frac{1}{m}\sum_{i=1}^{m}p_H(y_i\mid X_t,I,y_{<i}),
\qquad
\Delta A_t=A_t^{\mathrm{cur}}-A_t^{\mathrm{prev}},
\qquad
D_t=\frac{H_t-H_1}{|H_1|}.
\label{eq:state_signals}
\end{equation}
Here, $X_t$ denotes the research dialogue at checkpoint $t$, $I$ is a fixed
answer-probe suffix, $A_t$ denotes the attention assigned to retrieved evidence,
and $H_t$ denotes the aggregated attention entropy. Larger $C_t$ indicates
higher answer confidence, $\Delta A_t<0$ indicates that newly retrieved evidence
receives less attention than preceding evidence, and $D_t<\tau_H$ indicates
that the attention state has converged. Detailed definitions and computation are
provided in Appendix~\ref{con:light_signal}.

\paragraph{State-guided control.}
The three signals determine the trajectory-level control action. A
high-confidence provisional answer triggers \textsc{Stop}. When confidence
remains low but the attention state has converged, RIVO invokes
\textsc{Verify}. For a low-confidence and unconverged state, decreasing
evidence attention indicates an unproductive retrieval direction and triggers
\textsc{Reroute}. Otherwise, the current research trajectory
\textsc{Continue}s:
\begin{equation}
d_t=
\begin{cases}
\textsc{Stop}, &
C_t>\lambda,\\
\textsc{Verify}, &
C_t\leq\lambda \land D_t \leq\tau_H,\\
\textsc{Reroute}, &
C_t\leq\lambda \land D_t>\tau_H \land \Delta A_t<0,\\
\textsc{Continue}, &
\text{otherwise}.
\end{cases}
\label{eq:state_control}
\end{equation}

For low-confidence states initially assigned \textsc{Continue}, the trained
reroute predictor $P$ provides an additional learned check. The decision is promoted to \textsc{Reroute} when
$\hat{y}_t^{P}=\textsc{Reroute}$.
Once \textsc{Reroute} is triggered, the logit-intervention module is activated
to redirect subsequent retrieval.

% \subsubsection{Retrieval-Boundary State Estimation}
% \label{sec:asag_state}

\subsection{Logit Intervention}
\label{sec:logit_intervention}

Once \textsc{Reroute} is triggered by the above state-guided control mechanism,
RIVO activates the logit-intervention head attached to $P$. The head is trained
from downstream trajectory-level outcomes to learn residual corrections to the
Host token logits. During inference, the learned correction is applied while
the Host remains frozen, redirecting subsequent retrieval.

\paragraph{Low-rank logit correction.}
At reroute step $t$ and token position $s$, let
$\mathbf{z}_{t,s}^{H}\in\mathbb{R}^{|\mathcal{V}|}$ denote the
pre-softmax logits produced by the frozen Host, and let
$\mathbf{h}_{t,s}^{P}$ denote the hidden representation produced by the
plug-in. We project the plug-in representation through a low-rank
correction head,
\begin{equation}
\mathbf{r}_{t,s}
=
W_{\mathrm{down}}\mathbf{h}_{t,s}^{P},
\qquad
\Delta\mathbf{z}_{t,s}
=
W_{\mathrm{up}}\mathbf{r}_{t,s}.
\end{equation}
To keep the intervention sparse, we retain only the $K$
largest-magnitude entries of $\Delta\mathbf{z}_{t,s}$, where $t$ denotes the Deep Research turn at which rerouting is applied,
$s$ denotes the token position within the generated reroute query, and
$j$ indexes candidate tokens $x_j\in\mathcal{V}$ in the vocabulary. Let
$M_{t,s}$ denote the corresponding Top-$K$ mask, 
The corrected
Host logits and resulting token policy are
\begin{equation}
\mathbf{z}_{t,s}
=
\mathbf{z}^{H}_{t,s}
+
\beta
\left(
M_{t,s}\odot\Delta\mathbf{z}_{t,s}(\theta)
\right),
\qquad
\pi_{\beta,\theta}
\left(
x_j\mid\mathcal{F}_{t,s-1}
\right)
=
\operatorname{softmax}
\left(
\mathbf{z}_{t,s}
\right)_j .
\label{eq:logit_intervention}
\end{equation}
Here, $\beta$ controls the intervention strength. In this way, RIVO
changes only selected token preferences while leaving the Host parameters
unchanged.
% and the resulting token policy is
% \begin{equation}
% \pi_{\beta,\theta}
% \left(
% x_j\mid\mathcal{F}_{t,s-1}
% \right)
% =
% \operatorname{softmax}
% \left(
% \mathbf{z}_{t,s}
% \right)_j
% =
% \frac{
% \exp\!\left(
% z_{t,s,j}^{H}
% +
% \beta g_{t,s,j}(\theta)
% \right)
% }{
% \sum_{k\in\mathcal{V}}
% \exp\!\left(
% z_{t,s,k}^{H}
% +
% \beta g_{t,s,k}(\theta)
% \right)
% }.
% \label{eq:token_policy}
% \end{equation}
\paragraph{Theoretical justification.}
We provide a local analysis of how the learned logit intervention affects
the expected downstream value of a rerouting decision.
Specifically, we characterize the first-order change in trajectory-level
reward induced by perturbing the Host logits, and derive a sufficient
condition under which the intervention improves the expected downstream
value. Let $q_t=(x_{t,1},\ldots,x_{t,L_t})$ denote the complete reroute query
generated when rerouting is triggered at turn $t$. Here, $q_t$ denotes a
general textual representation of the rerouting direction; in RIVO, it is
instantiated as a structured rerouting hint that is subsequently rewritten by
the frozen Host into a search query.
By autoregressive factorization,
\begin{equation}
\pi_{\beta,\theta}
(q_t\mid\mathcal{F}_{t-1})
=
\prod_{s=1}^{L_t}
\pi_{\beta,\theta}
(x_{t,s}\mid\mathcal{F}_{t,s-1}).
\end{equation}
When $\beta=0$, the logit intervention vanishes and the policy reduces to
the original Host policy, i.e., $\pi_{0,\theta}=\pi_H$.
For a complete reroute query $q_t$, we define its expected downstream
research utility as
$
Q_t(q_t)
=
\mathbb{E}
\left[
R(\tau_{1:T})
\mid
q_t,\mathcal{F}_{t-1}
\right]
$.
The corresponding expected utility under intervention strength $\beta$ is
$
J_t(\beta)
=
\mathbb{E}_{q_t\sim
\pi_{\beta,\theta}(\cdot\mid\mathcal{F}_{t-1})}
\left[
Q_t(q_t)
\right].
$
We further define the first-order query-level probability shift as
\begin{equation}
S_t(\theta;q_t)
=
\left.
\frac{\partial}{\partial\beta}
\log
\pi_{\beta,\theta}
(q_t\mid\mathcal{F}_{t-1})
\right|_{\beta=0}.
\label{eq:query_shift}
\end{equation}
Then,
\begin{equation}
\left.
\frac{\partial J_t(\beta)}
{\partial\beta}
\right|_{\beta=0}
=
\operatorname{Cov}_{q_t\sim
\pi_H(\cdot\mid\mathcal{F}_{t-1})}
\left(
Q_t(q_t),
S_t(\theta;q_t)
\right).
\label{eq:first_order_value}
\end{equation}

\begin{theorem}[Local Improvement under Value-Aligned Logit Intervention]
\label{thm:local_improvement}

Let
\begin{equation}
\Gamma_t
:=
\operatorname{Cov}_{q_t\sim
\pi_H(\cdot\mid\mathcal F_{t-1})}
\left(
Q_t(q_t),
S_t(\theta;q_t)
\right).
\end{equation}

Suppose the intervention direction is locally aligned with downstream value,
such that $\beta\Gamma_t>0$, and there exists $B_J>0$ satisfying
$
|J_t''(u)|\leq B_J
$
for all $u$ between $0$ and $\beta$. Then
\begin{equation}
J_t(\beta)-J_t(0)
\geq
\beta\Gamma_t-\frac{B_J}{2}\beta^2.
\end{equation}

Consequently, for any nonzero $\beta$ satisfying
$
\beta\Gamma_t>0
$, $
0<|\beta|<\frac{2|\Gamma_t|}{B_J}
$, we have $
J_t(\beta)>J_t(0).
$
\end{theorem}

The theorem shows that when the intervention direction is aligned with
higher downstream value and its magnitude is sufficiently small, the
logit intervention locally improves the expected downstream value relative
to the frozen Host policy. Full assumptions, lemmas, and proofs are
provided in Appendix~\ref{proof}.
\subsection{Optimization and Inference}
\label{sec:optimization}

We optimize the lightweight plug-in using trajectory-level feedback while
keeping the Host model frozen. For each question $q$, we sample a group of
$G$ complete Deep Research trajectories,
$\mathcal{T}_q=\{\tau_1,\ldots,\tau_G\}$.
Whenever \textsc{Reroute} is triggered, the plug-in $P$ applies the logit
intervention described in Section~\ref{sec:logit_intervention}. Let
$\theta_P^{\mathrm{logit}}$ denote the trainable parameters of the
logit-correction head of $P$. At token position $s$, the corrected Host
logits are
$
\mathbf{z}_{t,s}
=
\mathbf{z}^{H}_{t,s}
+
\beta
\left(
M_{t,s}\odot
\Delta\mathbf{z}_{t,s}
(\theta_P^{\mathrm{logit}})
\right),
\label{eq:grpo_logit_update}
$
where the Host logits $\mathbf{z}^{H}_{t,s}$ remain fixed.
After each trajectory terminates, we use its downstream reward, as described in
Appendix~\ref{app:reward}, to compute a group-relative advantage $\hat A_i$ and optimize
$\theta_P^{\mathrm{logit}}$ with GRPO:
\begin{equation}
\mathcal{L}_{\mathrm{GRPO}}
=
-
\frac{1}{\sum_{i,j}m_{i,j}}
\sum_{i,j}
m_{i,j}
\min
\left(
\rho_{i,j}\hat A_i,\,
\operatorname{clip}
(\rho_{i,j},1-\epsilon,1+\epsilon)\hat A_i
\right),
\end{equation}
where $\rho_{i,j}$ is the policy ratio and $m_{i,j}$ selects trainable
rerouting tokens. Gradients are applied only to
$\theta_P^{\mathrm{logit}}$, while the Host remains frozen.
At inference time, the learned correction is activated only when
\textsc{Reroute} is triggered. The resulting rerouting hint is rewritten
by the frozen Host into a new search query, and the retrieved evidence is
fed back into the next research checkpoint.

%% file: 4_experiment.tex
% !TeX root = iclr2027_conference.tex
\section{Experiments}
\paragraph{Datasets.}
We construct \textsc{BrowseCore}, a curated Deep Research dataset paired with
a local retrieval corpus for plug-in training and controlled evaluation.
BrowseCore contains 2,870 questions for training, 359 for development, together with a   test set of 359 questions,
with each question associated with verified supporting evidence and hard-negative
documents. We evaluate RIVO on the 359 question test set. To assess
generalization beyond BrowseCore, we additionally evaluate on all 830
BrowseComp-Plus questions~\citep{chen2025browsecomp} and the 100-task
DRBench benchmark~\citep{abaskohi2026drbench}. Detailed BrowseCore
construction procedures are provided in Appendix~\ref{con:ourdataset}.

\paragraph{Baselines and evaluation setup.}
We evaluate RIVO on two representative Deep Research systems,
 LiteResearcher-4B \citep{li2026literesearcher} and Marco-DeepResearch-8B \citep{zhu2026marco}.
We compare against four baselines for inference-time trajectory control.
ReAct~\citep{yao2023react} interleaves reasoning with tool use through
prompting without an explicit trajectory controller.
ASAG~\citep{li2026stop} uses attention-state signals to adaptively
terminate unnecessary reasoning.
RAAC~\citep{soudani2026when} incorporates retrieval-aware signals to
detect stagnation and adapt long-horizon research trajectories.
PPDPP~\citep{deng2024plug} uses a lightweight plug-in to guide the
policy of a frozen language model.
For BrowseComp-Plus and BrowseCore, we follow the OpenResearcher evaluation
protocol and report Success Rate (SR) using a GPT-4-based judge for final-answer
correctness. For DRBench, we use the official \texttt{score\_report}
implementation with gpt-4o-mini as the judge model, and report
Insight Recall (IR), Factuality (Fact.), Report Quality (RQ), Distractor
Avoidance (DA), and Harmonic Mean (HM), with HM serving as the overall
aggregate metric. All methods use Qwen3-Embedding-8B~\citep{zhang2025qwen3} as the retriever for a controlled comparison.

\subsection{Main Results}

\begin{table*}[t]
\centering
\caption{
Main results on BrowseComp-Plus, BrowseCore, and DRBench.
SR denotes Success Rate. For DRBench, IR, Fact, DA, RQ, and HM denote
Insight Recall, Factuality, Distractor Avoidance, Report Quality, and
Harmonic Mean, respectively. Higher is better for all metrics.
}
\label{tab:main_results}

\small
\renewcommand{\arraystretch}{1.15}
\setlength{\tabcolsep}{5.5pt}

\resizebox{\textwidth}{!}{
\begin{tabular}{llccccccc}
\toprule
\textbf{Deep Research Model}
& \textbf{Method}
& \multicolumn{1}{c}{\textbf{BrowseComp-Plus}}
& \multicolumn{1}{c}{\textbf{BrowseCore}}
& \multicolumn{5}{c}{\textbf{DRBench}} \\

\cmidrule(lr){3-3}
\cmidrule(lr){4-4}
\cmidrule(lr){5-9}

& & \textbf{SR $\uparrow$} & \textbf{SR $\uparrow$} & \textbf{IR $\uparrow$} & \textbf{Fact $\uparrow$} & \textbf{DA $\uparrow$} & \textbf{RQ $\uparrow$} & \textbf{HM $\uparrow$} \\

\midrule

% =========================================================
% LiteResearcher-4B
% =========================================================
\multirow{6}{*}{\textbf{LiteResearcher-4B}}

& Source & 32.29 & 53.48 & 5.24 & 54.47 & 99.90 & 81.23 & 17.27 \\
& ASAG & 33.61 & 54.60 & 11.76 & 87.00 & 99.57 & 76.60 & 33.43 \\
& RAAC & 33.73 & 54.04 & 6.41 & 51.52 & 99.63 & 81.89 & 20.23 \\
& PPDPP & 32.65 & 55.15 & 3.34 & 73.14 & 99.63 & 76.74 & 11.91 \\
& ReAct & 31.57 & 54.32 & 8.76 & 51.34 & 99.60 & 80.79 & 25.64 \\
& \textsc{RIVO(ours)}    & \textbf{36.14} & \textbf{59.30} & 12.61 & 81.85 & 99.54 & 76.36 & \textbf{34.88} \\

\midrule

% =========================================================
% Marco-DeepResearch-8B
% =========================================================
\multirow{6}{*}{\textbf{Marco-DeepResearch-8B}}

& Source & 31.70 & 48.75 & 15.95 & 82.49 & 99.08 & 72.74 & 40.54 \\
& ASAG & 29.78 & 50.42 & 12.43 & 74.40 & 99.24 & 74.90 & 34.09 \\
& RAAC & 32.53 & 52.00 & 15.21 & 79.00 & 98.82 & 71.58 & 39.02 \\
& PPDPP & 33.37 & 50.97 & 15.22 & 79.81 & 98.97 & 72.52 & 39.16 \\
& ReAct &32.90& 50.14 & 16.09 & 79.52 & 99.03 & 72.22 & 40.53 \\
& \textsc{RIVO (ours) } & \textbf{33.98} & \textbf{52.40} & 18.56 & 85.55 & 99.29 & 74.40 & \textbf{44.90} \\

\bottomrule
\end{tabular}
}
\end{table*}

 Table~\ref{tab:main_results} compares RIVO with
ReAct~\citep{yao2023react}, ASAG~\citep{li2026stop},
RAAC~\citep{soudani2026when}, and PPDPP~\citep{deng2024plug}
across three benchmarks and two Deep Research models.
RIVO achieves the strongest overall result across all six model--benchmark
settings. For example, with LiteResearcher-4B, it improves over the strongest
baseline by 2.41 points on BrowseComp-Plus, 4.15 points on BrowseCore, and
1.45 points in DRBench HM. The gains remain consistent when switching to
Marco-DeepResearch-8B, where RIVO also improves over the strongest baseline
across all three benchmarks. This cross-model pattern indicates that the
benefit of RIVO is not tied to a particular Host architecture, although the
magnitude of improvement varies across evaluation settings.
%%%%
We find that Marco-DeepResearch-8B exhibits a different performance profile
from LiteResearcher-4B. Despite its larger scale, Marco-DeepResearch-8B performs worse than LiteResearcher-4B on BrowseComp-Plus and BrowseCore, but substantially better on DRBench. This cross-benchmark variation indicates that the two Hosts exhibit different performance profiles, suggesting that benchmark characteristics and model-specific research behaviors may interact in nontrivial ways. On DRBench, RIVO does not achieve the highest score on every individual metric, but it achieves the highest Harmonic Mean for both Host models. With LiteResearcher-4B, RIVO reaches an HM of 34.88, compared with 33.43 for the strongest baseline, while with Marco-DeepResearch-8B, it achieves 44.90 versus 40.54. 
These results show that RIVO achieves stronger overall performance through more effective trajectory correction.

% The DRBench results further show that the improvement is not simply due to
% optimizing a single report-quality dimension. On Marco-DeepResearch-8B,
% RIVO substantially increases both insight recall and factuality while
% maintaining strong distractor avoidance and report quality, leading to the
% largest HM gain. On LiteResearcher-4B, RIVO similarly improves the aggregate
% HM despite not attaining the highest score on every individual dimension.
% Together, these results suggest that state-conditioned rerouting combined
% with learned logit intervention provides a robust mechanism for improving
% long-horizon research trajectories across different Hosts and evaluation
% objectives.

\subsection{Further Analysis}

In this section, we systematically analyze the key design choices of RIVO,
including the intervention strength $\beta$, correction scope $K$,
state-control thresholds $\lambda$ and $\tau_H$, the contribution of individual
state signals, and the rerouting design. Additional ablations in
Appendix~\ref{ablation} further examine the two functions of the 0.6B plug-in
and compare alternative controller configurations, including STOP-only,
REROUTE-only, standalone three-way control, and the variant without
\textsc{Verify}.

\begin{wrapfigure}{r}{0.42\textwidth}
    \centering
    \vspace{-8pt}
    \includegraphics[width=\linewidth]{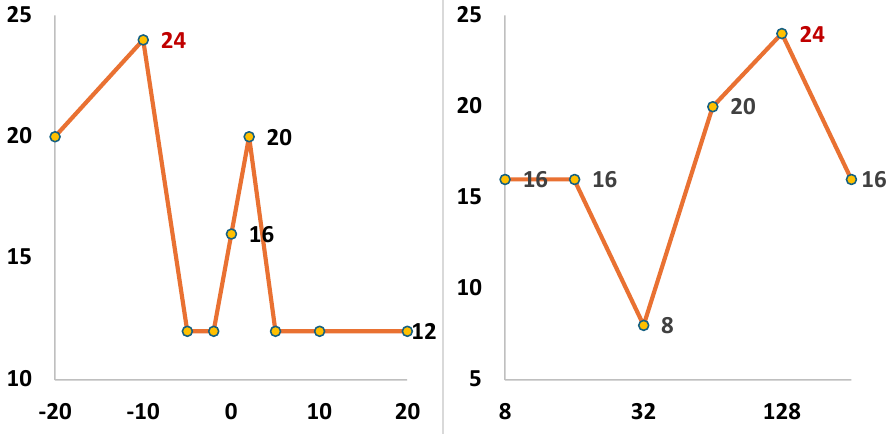}
    \caption{
    Sensitivity analysis of RIVO under different hyperparameter settings.
    }
    \label{fig:sensitivity}
    \vspace{-8pt}
\end{wrapfigure}

\paragraph{Sensitivity to Intervention Strength and Scope.}
There are two important hyperparameters in RIVO, the intervention strength
$\beta$ and the correction scope $K$. The parameter $\beta$ controls the
magnitude of the residual logit correction applied to the frozen Host,
while $K$ determines how many logits are modified during rerouting.
Together, they control how strongly and how broadly RIVO redirects the
current research trajectory. An intervention that is too weak may fail to
change the retrieval direction, whereas an overly strong or overly broad
correction may unnecessarily perturb the Host distribution.
We therefore study the sensitivity of RIVO to both parameters on the
BrowseCore test set using LiteResearcher-4B with a fixed
Qwen-Embedding-8B retriever. For questions where rerouting is triggered,
we replay the trajectory from the same pre-reroute state while varying
$\beta\in\{-20,-10,-5,-2,0,2,5,10,20\}$. The question, trajectory prefix,
checkpoint state, retrieval context, and random seed are kept fixed before
rerouting, while the reroute query and all subsequent research steps are
regenerated. As shown in Figure~\ref{fig:sensitivity}, performance varies
substantially with the intervention strength. The best result is obtained
at $\beta=-10$, achieving a success rate of 24, compared with 16
when $\beta=0$.
We further vary
$K\in\{8,16,32,64,128,256\}$ to examine the effect of intervention scope.
Performance also changes noticeably across different values of $K$, with
the best result achieved at $K=128$ with a success rate of 24.
These results demonstrate that the learned logit intervention in RIVO is effective,
with performance improving under appropriately scaled and selective corrections
to the Host logits.

\paragraph{Sensitivity to State-Control Thresholds.}

\begin{wrapfigure}{r}{0.42\linewidth}
    \centering
    \vspace{-8pt}
    \includegraphics[width=\linewidth]{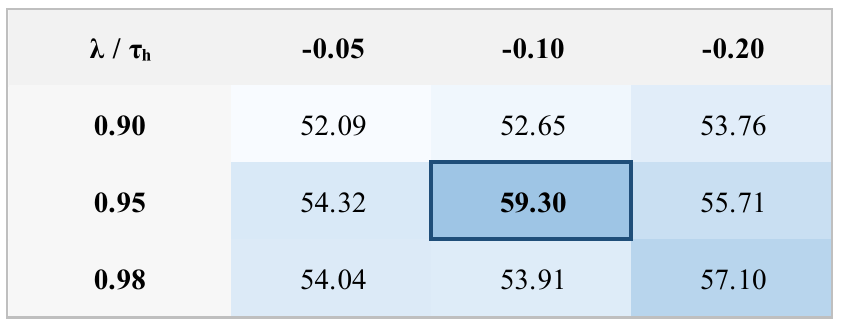}
    \caption{Sensitivity to the state-control thresholds $\lambda$ and $\tau_H$.
    Higher SR is better.}
    \label{fig:state_threshold}
    \vspace{-10pt}
\end{wrapfigure}

We further analyze the two thresholds used by the state-guided controller.
The confidence threshold $\lambda$ determines whether the current trajectory
is considered sufficiently reliable, while the convergence threshold
$\tau_H$ measures whether the research state has stabilized.
We evaluate different combinations of
$\lambda\in\{0.90,0.95,0.98\}$ and
$\tau_H\in\{-0.05,-0.10,-0.20\}$.
As shown in Figure~\ref{fig:state_threshold}, the best performance is obtained
with $\lambda=0.95$ and $\tau_H=-0.10$, achieving an SR of 59.30. A low threshold such as $\lambda=0.90$ may
over-trust uncertain provisional answers, while
a higher threshold requires stronger confidence before termination. The best
performance at $\lambda=0.95$ suggests that a moderate confidence criterion
provides a better balance between premature stopping and unnecessary continued
research.

\paragraph{Ablation of State Signals.}
\begin{wrapfigure}{r}{0.4\linewidth}
    \centering
    \vspace{-8pt}
    \includegraphics[width=\linewidth]{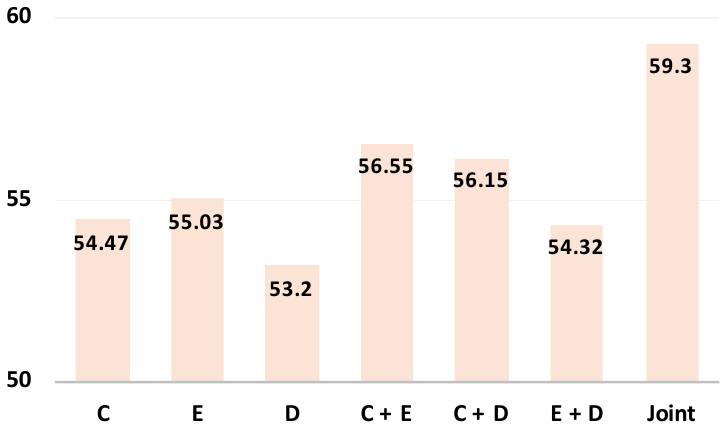}
    
    \caption{
    Ablation of state signals.
    }
    \vspace{-12pt}
    \label{fig:state_signal}
    \vspace{-8pt}
\end{wrapfigure}
We analyze the contribution of the three state signals used by RIVO in
Figure~\ref{fig:state_signal}: confidence $C^t$, evidence attention $A^t$,
and convergence $D^t$, denoted as C, E, and D in the figure, respectively.
These signals capture complementary aspects of the current research state:
$C^t$ reflects the reliability of the current answer, $A^t$ measures the
attention assigned to retrieved evidence, and $D^t$ indicates whether the
research trajectory has converged.
Among the individual signals, Evidence Attention $A^t$ performs best with
an SR of 55.03, followed by Confidence $C^t$ at 54.47 and
Convergence $D^t$ at 53.20. Pairwise combinations generally improve
performance, with $\{C^t,A^t\}$ achieving the strongest two-signal result
of 56.55 and $\{C^t,D^t\}$ reaching 56.15. In contrast,
$\{A^t,D^t\}$ obtains 54.32, suggesting that confidence provides an
important complementary signal when combined with either evidence-based
or convergence-based information.
Using all three signals achieves the best performance of 59.30, outperforming
the best single-signal and two-signal settings by 4.27 and 2.75 points,
respectively. This shows that the three signals provide complementary
information for state-guided trajectory control.

\begin{wraptable}{r}{0.46\linewidth}
    \centering
    \vspace{-8pt}
    \small
    \setlength{\tabcolsep}{4pt}
    \renewcommand{\arraystretch}{1.05}
    \begin{tabular}{lc}
        \toprule
        \textbf{Variant} & \textbf{SR (\%)} \\
        \midrule
        \multicolumn{2}{l}{\textit{Plug-in Function}} \\
        w/o Reroute Predictor (Pre) & 53.63 \\
        w/o Logit Intervention & 56.82 \\
        w/o Logit Intervention+Reroute Pre &55.71\\
        Full RIVO & \textbf{59.30} \\
        \midrule
        \multicolumn{2}{l}{\textit{Rerouting Representation}} \\
        Direct Question & 54.87 \\
        Structured Hint & \textbf{59.30} \\
        \bottomrule
    \end{tabular}
    \caption{Ablation of the rerouting design.}
    \label{tab:rerouting_ablation}
    \vspace{-8pt}
\end{wraptable}

\paragraph{Ablation of Rerouting Design.}
The 0.6B plug-in in RIVO has two distinct functions: a \emph{reroute
predictor}, which determines whether a current \textsc{Continue} state should
be promoted to \textsc{Reroute}, and a \emph{logit-intervention module}, which
modifies the frozen Host logits after rerouting is triggered to steer the next
retrieval direction.
In \textit{w/o Reroute Predictor}, the learned reroute prediction is disabled,
so rerouting can only be triggered by the state-guided controller, while the
logit intervention is retained. In \textit{w/o Logit Intervention}, the reroute
predictor is retained, but the residual correction to the Host logits is removed,
so the Host generates the subsequent rerouting direction without learned logit
modification. In \textit{w/o Logit Intervention + Reroute Predictor}, both
learned components are removed, leaving only the state-guided controller for
trajectory control.
As shown in Table~\ref{tab:rerouting_ablation}, removing the reroute predictor
reduces SR from 59.30 to 53.63, while removing the logit intervention reduces
SR to 56.82. These results show that both deciding \emph{when} to introduce an
additional reroute and learning \emph{how} to redirect the frozen Host
contribute to the final performance.

%% file: 6_conclution.tex
\section{Conclusion}

We present \textsc{RIVO}, a lightweight framework for adaptive trajectory
correction in long-horizon Deep Research. RIVO combines state-guided control with a lightweight learned reroute predictor and token-level logit intervention, enabling adaptive decisions about when trajectory redirection is needed and how the frozen Host should be redirected without retraining the Host model. We also construct \textsc{BrowseCore}, a
curated question--answer dataset paired with a compact local retrieval corpus
for controlled training and evaluation of Deep Research agents. Experiments on BrowseComp-Plus, BrowseCore, and DRBench with two different Deep Research Hosts demonstrate performance gains over existing trajectory-control baselines across all evaluated settings, while ablation studies show that the state signals, reroute prediction, and logit intervention each contribute to the overall performance. Our theoretical analysis further shows that,
under a value-alignment condition and bounded local curvature, a sufficiently
small logit intervention can locally improve the expected downstream research
utility. Overall, these results demonstrate that lightweight, state-conditioned
logit intervention provides an effective approach to improving long-horizon
Deep Research trajectories.
% \section{Limitations and Future Directions}
% \label{sec:limitations}
\section{Limitations and Future Directions}

The current framework uses relatively constrained rerouting decisions and
separates state-based control from the learned logit intervention. Future work
could explore more adaptive rerouting strategies that support multiple
interventions over a trajectory and jointly learn when to intervene, how
strongly to intervene, and which parts of the Host distribution should be
modified. Another promising direction is to make the retriever itself adaptive,
allowing the retrieval model and trajectory controller to be jointly optimized
from downstream feedback so that retrieval representations and rerouting
strategies can co-adapt throughout the research process.

%% file: 9_appendix_navigation.tex
% =========================================================
% Appendix Navigation
% =========================================================
% \clearpage
% \section*{Appendix Navigation}

% =========================================================
% A. RIVO's Adaptive Rerouting Pipeline
% =========================================================
% =========================================================
% Appendix Navigation
% =========================================================
\clearpage
\section*{Appendix Navigation}

% =========================================================
% A. Ablation Study
% =========================================================
\noindent
\hyperref[ablation]{%
\textbf{\ref*{ablation}\quad Ablation Study}}
\dotfill
\pageref*{ablation}
\par\vspace{4pt}

% =========================================================
% B. RIVO's Adaptive Rerouting Pipeline
% =========================================================
\noindent
\hyperref[app:rivo]{%
\textbf{\ref*{app:rivo}\quad RIVO's Adaptive Rerouting Pipeline}}
\dotfill
\pageref*{app:rivo}
\par\vspace{4pt}

% B.1
\hspace*{1.5em}
\hyperref[con:abliccation_stusy]{%
\ref*{con:abliccation_stusy}\quad Structural Ablations}
\dotfill
\pageref*{con:abliccation_stusy}
\par\vspace{2pt}

% B.2
\hspace*{1.5em}
\hyperref[con:light_signal]{%
\ref*{con:light_signal}\quad Lightweight State Signals}
\dotfill
\pageref*{con:light_signal}
\par\vspace{4pt}

% =========================================================
% B. Rerouting Reward Details
% =========================================================
\noindent
\hyperref[app:reward]{%
\textbf{\ref*{app:reward}\quad Rerouting Reward Details}}
\dotfill
\pageref*{app:reward}
\par\vspace{4pt}

% =========================================================
% C. Implementation Details
% =========================================================
\noindent
\hyperref[im_details]{%
\textbf{\ref*{im_details}\quad Implementation Details}}
\dotfill
\pageref*{im_details}
\par\vspace{4pt}

% =========================================================
% D. Turn-Level Supervision and Data Construction
% =========================================================
\noindent
\hyperref[trajactory_label]{%
\textbf{\ref*{trajactory_label}\quad Turn-Level Supervision and Data Construction}}
\dotfill
\pageref*{trajactory_label}
\par\vspace{4pt}

% =========================================================
% E. Prompt
% =========================================================
\noindent
\hyperref[app:prompt]{%
\textbf{\ref*{app:prompt}\quad Prompt}}
\dotfill
\pageref*{app:prompt}
\par\vspace{3pt}

% E.1
\hspace*{1.5em}
\hyperref[con:systempromt]{%
\ref*{con:systempromt}\quad Deep Research System Prompt}
\dotfill
\pageref*{con:systempromt}
\par\vspace{2pt}

% E.2
\hspace*{1.5em}
\hyperref[con:browser_tools]{%
\ref*{con:browser_tools}\quad LiteResearcher-4B Browser Tools}
\dotfill
\pageref*{con:browser_tools}
\par\vspace{2pt}

% E.3
\hspace*{1.5em}
\hyperref[con:turn-level prompt]{%
\ref*{con:turn-level prompt}\quad Turn-Level Annotation for Plug-in Training}
\dotfill
\pageref*{con:turn-level prompt}
\par\vspace{2pt}

% E.4
\hspace*{1.5em}
\hyperref[con:evidence document prompt]{%
\ref*{con:evidence document prompt}\quad Evidence Document Verification}
\dotfill
\pageref*{con:evidence document prompt}
\par\vspace{2pt}

% E.5
\hspace*{1.5em}
\hyperref[clue promt]{%
\ref*{clue promt}\quad Prompt Template for Clue Decomposition}
\dotfill
\pageref*{clue promt}
\par\vspace{4pt}

% =========================================================
% F. BrowseCore
% =========================================================
\noindent
\hyperref[con:ourdataset]{%
\textbf{\ref*{con:ourdataset}\quad BrowseCore}}
\dotfill
\pageref*{con:ourdataset}
\par\vspace{3pt}

% F.1
\hspace*{1.5em}
\hyperref[con:motivation]{%
\ref*{con:motivation}\quad Background and Motivation}
\dotfill
\pageref*{con:motivation}
\par\vspace{2pt}

% F.2
\hspace*{1.5em}
\hyperref[con:corpus contruction]{%
\ref*{con:corpus contruction}\quad Benchmark and Corpus Construction}
\dotfill
\pageref*{con:corpus contruction}
\par\vspace{2pt}

% F.2.1
\hspace*{3em}
\hyperref[con:question source]{%
\ref*{con:question source}\quad Question Source and Filtering}
\dotfill
\pageref*{con:question source}
\par\vspace{2pt}

% F.2.2
\hspace*{3em}
\hyperref[con:document source]{%
\ref*{con:document source}\quad Candidate Document Sources}
\dotfill
\pageref*{con:document source}
\par\vspace{2pt}

% F.2.3
\hspace*{3em}
\hyperref[con:Evidence source]{%
\ref*{con:Evidence source}\quad Evidence Collection and Verification}
\dotfill
\pageref*{con:Evidence source}
\par\vspace{2pt}

% F.2.4
\hspace*{3em}
\hyperref[con:pos document]{%
\ref*{con:pos document}\quad Positive and Gold Document Identification}
\dotfill
\pageref*{con:pos document}
\par\vspace{2pt}

% F.2.5
\hspace*{3em}
\hyperref[con:hard negative]{%
\ref*{con:hard negative}\quad Hard-Negative Mining}
\dotfill
\pageref*{con:hard negative}
\par\vspace{2pt}

% F.2.6
\hspace*{3em}
\hyperref[con:final]{%
\ref*{con:final}\quad Final Corpus Assembly}
\dotfill
\pageref*{con:final}
\par\vspace{4pt}

\noindent
\hyperref[sec:additional-related-work]{%
\textbf{\ref*{sec:additional-related-work}\quad Additional Related Work}}
\dotfill
\pageref*{sec:additional-related-work}
\par
% =========================================================
% G. Proof of Main Theorem
% =========================================================
\noindent
\hyperref[proof]{%
\textbf{\ref*{proof}\quad Proof of Main Theorem}}
\dotfill
\pageref*{proof}
\par

\clearpage

%% file: 8_appendix.tex
% =========================
% Add to the preamble
% =========================
% =========================
% Preamble
% =========================
% !TeX root = iclr2027_conference.tex

\section{Ablation study}

\label{ablation}
\paragraph{Ablation of the 0.6B Plug-in.}

The 0.6B plug-in in RIVO is used in two distinct stages: first, to predict whether a current \textsc{Continue} state should be promoted to \textsc{Reroute}, and second, to modify the frozen Host logits after rerouting is triggered. We separately remove each function while keeping the remaining pipeline unchanged. Removing the 0.6B reroute predictor reduces SR from 59.30 to 53.63, indicating that the learned reroute decision contributes substantially to identifying useful intervention points. When the reroute decision is retained but the 0.6B logit intervention is removed, the system directly outputs the rerouting hint without modifying the Host logits and achieves an SR of 56.82. The degradation in both variants shows that the two uses of the 0.6B plug-in are complementary: the predictor helps determine \emph{when} rerouting should occur, while the logit intervention controls \emph{how} the frozen Host is redirected toward a more effective retrieval trajectory.

\paragraph{Controller Design Ablation.}

\begin{wraptable}{r}{0.38\textwidth}
\vspace{-8pt}
\centering
\caption{Ablation study of different controller designs.}
\label{tab:controller_ablation}
\small
\setlength{\tabcolsep}{5pt}
\begin{tabular}{lc}
\toprule
\textbf{Controller Design} & \textbf{SR (\%)} \\
\midrule
STOP-only             & 56.27 \\
REROUTE-only           & 55.15 \\
Standalone 0.6B 3-way & 47.91 \\
w/o Verify        &   54.32\\
Joint controller (Ours)          & \textbf{59.30} \\
\bottomrule
\end{tabular}
\vspace{-8pt}
\end{wraptable}

We further examine how different controller designs affect long-horizon
research performance. Table~\ref{tab:controller_ablation} compares variants
that restrict the lightweight controller to a single action, allow it to
independently perform full three-way control, or combine it with the original
control mechanism.
The definitions of \textsc{STOP}-only, \textsc{REROUTE}-only, and the
standalone 0.6B three-way controller are provided in
Section~\ref{con:abliccation_stusy}. \textit{w/o Verify} removes the
\textsc{Verify} action from the state-guided controller.

\section{\textbf{RIVO}'s adaptive rerouting pipeline}
\label{app:rivo}

\subsection{Structural Ablations.}
\label{con:abliccation_stusy}
Figure~\ref{fig:structural_ablation} summarizes the control flow of RIVO
and the structural variants used in our ablation study.
In the full model (\textbf{1}), RIVO first evaluates the current research
state using three complementary signals: confidence $C_t$, convergence
$D_t$, and evidence attention $A_t$. Confidence determines whether the
current trajectory is sufficiently reliable to terminate. For
low-confidence states, convergence identifies whether the current
provisional answer should be further verified, while evidence attention
compares the usefulness of newly retrieved evidence with previously
collected evidence to determine whether the trajectory should be rerouted
or continued. If the state-guided controller selects \textsc{Continue},
a lightweight 0.6B reroute predictor performs an additional learned check
and may promote the decision to \textsc{Reroute}. Once
\textsc{Reroute} is triggered, either by the state-guided controller or
the reroute predictor, the 0.6B plug-in applies token-level logit
correction to the frozen Host, guiding it toward a new retrieval direction.
A \textsc{Continue} decision instead leaves the Host trajectory unchanged
and proceeds to the next research step.
\begin{figure*}[t]
    \centering
    \includegraphics[width=1\columnwidth]{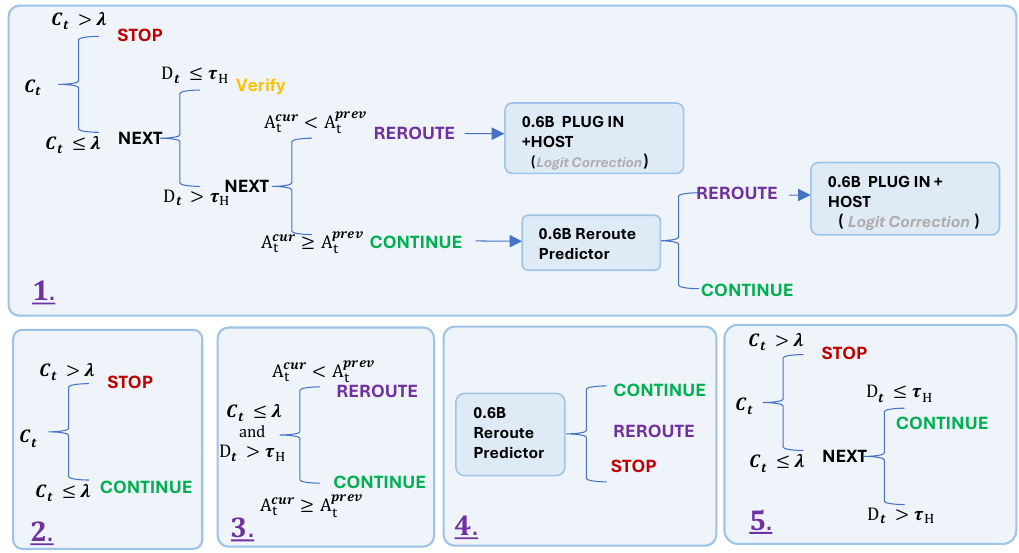}
    \caption{
    Structural ablations of \textbf{RIVO}.
    }
    \label{fig:structural_ablation}
    \vspace{-10pt}
\end{figure*}
We construct four structural variants to isolate different parts of this
control mechanism. Variant \textbf{(2)}, \textbf{STOP-only}, retains only
the confidence-based stopping mechanism. Variant \textbf{(3)},
\textbf{REROUTE-only}, retains the rule-based rerouting mechanism while
removing the other control branches. Variant \textbf{(4)},
\textbf{Predictor-only}, replaces the state-guided controller with the
lightweight 0.6B reroute predictor. Variant \textbf{(5)},
\textbf{w/o Verify}, removes the verification branch while preserving the
remaining control logic. Importantly, these variants modify only how the
control decision is made. After a \textsc{Reroute} or \textsc{Continue}
decision is produced, variants \textbf{(2)--(5)} follow exactly the same
downstream execution as the full model: \textsc{Reroute} invokes the
plug-in-based logit correction followed by a new retrieval step, whereas
\textsc{Continue} proceeds with the normal Host research trajectory.

When \textsc{Verify} is selected, the current provisional answer is
explicitly treated as an unverified hypothesis. The Host receives the
following instruction:
\begin{quote}
\small
\textit{The current provisional answer is: \{answer\}. Treat it as an
unverified hypothesis. Check it against the retrieved primary evidence,
look for contradictions, and use
\texttt{browser.open}/\texttt{browser.find} (or a new search when
necessary) to resolve any missing fact before answering.}
\end{quote}

A \textsc{Stop} decision terminates the research process and requires the
Host to return the final response in the format
\texttt{<answer>\textbackslash boxed\{answer\}</answer>}.
We additionally impose a maximum research-turn budget. If this budget is
exhausted before a normal stopping decision is reached, further tool calls
are disabled and the Host is instructed to synthesize the best-supported
answer from the evidence already collected:
\begin{quote}
\small
\textit{The research budget is exhausted. Do not call any tool. Synthesize
the best-supported final answer from the evidence already in this
conversation.  Return only \texttt{\textless{}answer\textgreater{}...\textless{}/answer\textgreater{}}.}
\end{quote}

\subsection{State-guided control}
\label{con:light_signal}
\paragraph{Provisional-answer confidence.}
At checkpoint $t$, let $X_t$ denote the serialized dialogue. We append a fixed answer-probe suffix $I$
directly at the token level and greedily decode a provisional answer
$Y_t=(y_1,\ldots,y_m)$. From the selected-token log-probabilities
$\ell_i=\log p_H(y_i\mid X_t,I,y_{<i})$, we compute
\begin{equation}
C_t
=
\frac{1}{m}\sum_{i=1}^{m}\exp(\ell_i),
\qquad
\mathrm{high}_t=\mathbf{1}[C_t>\lambda],
\end{equation}
where $m$ ends at the first complete boxed answer and $\lambda=0.95$.
$C_t$ provides a lightweight internal estimate of the Host's confidence
in its current provisional answer given the evidence available at the
current retrieval state.
\paragraph{Evidence-attention score.}
We monitor the final $L=4$ Host layers over a $W=32$ token window.
Let $a^{(l)}_{h,q,k}$ denote the attention weight from query position $q$
to token $k$ at head $h$ of layer $l$. Since a retrieval step may return
multiple pieces of evidence, we define $S_t$ as the union of all evidence-token
spans accumulated within the current retrieval segment before checkpoint $t$.
For an evidence segment $S$, we compute
\begin{equation}
A_t(S)
=
\frac{1}{L}
\sum_{l=1}^{L}
\frac{1}{N_h Q |S|}
\sum_h
\sum_q
\sum_{k\in S}
a^{(l)}_{h,q,k},
\label{eq:evidence_attention}
\end{equation}
where $N_h$ is the number of attention heads, $Q$ is the number of query
positions in the monitored window, and $|S|$ is the number of evidence
tokens. Thus, $A_t(S)$ measures the average attention assigned by the Host
to the evidence accumulated in a retrieval segment.
We denote the scores of the previous and current retrieval segments as
$A_t^{\mathrm{prev}}=A_t(S_{t-1})$ and
$A_t^{\mathrm{cur}}=A_t(S_t)$, respectively. Their difference,
$\Delta A_t=A_t^{\mathrm{cur}}-A_t^{\mathrm{prev}}$, indicates whether the
newly accumulated evidence receives more or less attention than the evidence
from the preceding retrieval segment.

\paragraph{Attention convergence.}
We use attention entropy to measure whether the Host has reached a more
focused and stable research state. For query position $q$, head $h$, and
layer $l$, we compute the normalized attention entropy
\begin{equation}
E^{(l)}_{h,q}
=
\frac{
-\sum_{k=1}^{K}
\widetilde a^{(l)}_{h,q,k}
\log \widetilde a^{(l)}_{h,q,k}
}{
\max(\log K,1)
},
\qquad
\widetilde a=\max(a,10^{-12}),
\label{eq:attention_entropy}
\end{equation}
where lower entropy indicates that attention is concentrated on a smaller
set of tokens.
We aggregate the entropy across the monitored layers, heads, and query
positions as
\begin{equation}
H_t
=
\sum_{l=1}^{L}\sum_h\sum_q E^{(l)}_{h,q},
\qquad
D_t
=
\frac{H_t-H_1}
{\max(|H_1|,10^{-12})},
\label{eq:entropy_delta}
\end{equation}
where $H_1$ is the entropy at the first retrieval checkpoint. Thus,
$D_t$ measures how much the attention distribution has become more
concentrated relative to the initial research state. We treat the state
as converged when
$
D_t<\tau_H
$,
with $\tau_H=-0.10$.

At each retrieval checkpoint, rerouting can be triggered by either the
inference-time controller or the learned 0.6B plug-in. The controller
triggers \textsc{Reroute} when the current provisional answer has low
confidence, the attention state has not converged, and the newly retrieved
evidence receives less attention than the preceding evidence. If the
controller instead returns a low-confidence \textsc{Continue}, we consult
the plug-in and promote the decision to \textsc{Reroute} when its rerouting
probability exceeds a threshold.The final control decision combines the state-guided controller and the
learned plug-in trigger:
\begin{equation}
\begin{aligned}
g_t^{\mathrm{stop}}
&=
\mathbf{1}\!\left[
C_t>\lambda
\right],\\
g_t^{\mathrm{verify}}
&=
\mathbf{1}\!\left[
C_t\leq\lambda
\land D_t\leq\tau_H
\right],\\
g_t^{\mathrm{ctrl}}
&=
\mathbf{1}\!\left[
C_t\leq\lambda
\land D_t>\tau_H
\land \Delta A_t<0
\right],\\
g_t^{\mathrm{plug}}
&=
\mathbf{1}\!\left[
d_t=\textsc{Continue}
\land C_t\leq\lambda
\land p_P(\textsc{Reroute}\mid s_t)\geq\tau_P
\land G_t=1
\right],\\
g_t^{\mathrm{reroute}}
&=
g_t^{\mathrm{ctrl}}
\lor
g_t^{\mathrm{plug}},\\
d_t^{\mathrm{final}}
&=
\begin{cases}
\textsc{Stop}, & g_t^{\mathrm{stop}}=1,\\
\textsc{Verify}, & g_t^{\mathrm{verify}}=1,\\
\textsc{Reroute}, & g_t^{\mathrm{reroute}}=1,\\
\textsc{Continue}, & \text{otherwise}.
\end{cases}
\end{aligned}
\label{eq:guarded_control}
\end{equation}

Here, $\Delta A_t=A_t^{\mathrm{cur}}-A_t^{\mathrm{prev}}$,
 and $G_t$ indicates whether
the minimum-round and cooldown constraints are satisfied.
When $g_t=1$, the plug-in correction head is activated to redirect the
subsequent retrieval.

\section{Rerouting Reward Details}
\label{app:reward}

We use a trajectory-level reward that combines final task success with the quality of the rerouting intervention. For trajectory $\tau_i$, the reward is
\begin{equation}
R_i
=
R_i^{\mathrm{outcome}}
+
\lambda Q_i
-
\gamma P_i,
\label{eq:app_reward}
\end{equation}
where $R_i^{\mathrm{outcome}}$ measures final-answer correctness, $Q_i$ measures rerouting-hint quality, and $P_i$ penalizes invalid rerouting outputs.

\paragraph{Outcome reward.}
The outcome reward is determined by the final answer of the complete Deep Research trajectory:
\begin{equation}
R_i^{\mathrm{outcome}}
=
\begin{cases}
1, & \text{if the final answer is judged correct},\\
0, & \text{otherwise}.
\end{cases}
\end{equation}

\paragraph{Rerouting-hint quality.}
Each rerouting hint consists of three fields:
\textsc{Entity}, \textsc{Relation}, and \textsc{Disambiguator}.
We evaluate the generated hint according to seven criteria:

\begin{enumerate}
    \item \textbf{Format validity}: the output follows the required three-field format;
    \item \textbf{English usage}: the generated fields are written in English;
    \item \textbf{No self-talk}: the output does not contain reasoning traces, instructions, or meta-language;
    \item \textbf{Entity anchoring}: the \textsc{Entity} field is grounded in the question or the observed research history;
    \item \textbf{Relation validity}: the \textsc{Relation} field provides a meaningful relation for redirecting retrieval;
    \item \textbf{Disambiguation validity}: the \textsc{Disambiguator} provides additional information that narrows the search direction;
    \item \textbf{Field distinctness}: the three fields provide non-redundant information.
\end{enumerate}

Let $c_{i,k}\in\{0,1\}$ indicate whether the $k$-th criterion is satisfied. We compute
\begin{equation}
S_i
=
\sum_{k=1}^{7}
w_k c_{i,k},
\qquad
\sum_{k=1}^{7} w_k = 1,
\end{equation}
where $w_k$ denotes the fixed importance weight of criterion $k$. We use uniform weights \(w_k = 1/7\) for all seven criteria.

The rerouting-quality reward is
\begin{equation}
Q_i
=
\begin{cases}
1,
& \text{if all criteria are satisfied},\\
-1+0.5S_i,
& \text{if the generated hint violates any criterion},\\
0,
& \text{if no rerouting hint is generated}.
\end{cases}
\label{eq:app_hint_quality}
\end{equation}

Thus, valid hints receive a positive reward, while invalid hints receive a graded negative reward according to the satisfied quality criteria.

\paragraph{Invalid-hint penalty.}
We additionally define
\begin{equation}
P_i
=
\begin{cases}
1, & \text{if a generated rerouting hint is invalid},\\
0, & \text{otherwise}.
\end{cases}
\end{equation}
No invalid-hint penalty is applied when no rerouting hint is generated.

For the configuration used in our main experiments, we set
$ \lambda=2.0,
\gamma=2.0.$

When no rerouting hint is generated, $Q_i=P_i=0$, and the reward reduces to  $ R_i=R_i^{\mathrm{outcome}}$

\section{implementation details}
\label{im_details}

For action prediction, we fine-tuned Qwen3-0.6B-Base using LoRA-based supervised fine-tuning (SFT) with two action labels: \textsc{Continue},   and \textsc{Reroute}. We set the maximum sequence length to 8,192 tokens and trained for 3 epochs using AdamW with a learning rate of \(2\times10^{-4}\), a cosine learning-rate schedule, and a warmup ratio of 0.03. For LoRA, we used a rank of 16, a scaling factor of 32, and a dropout rate of 0.05. Training was conducted on 4 GPUs with a per-device batch size of 1 and 4 gradient accumulation steps, resulting in an effective batch size of 16. To alleviate class imbalance among the action labels, we applied class-aware sampling with an exponent of 0.5.

We trained a Qwen3-0.6B control-and-correction plugin while keeping the Deep Research Host frozen. The plugin was initialized from a two-class SFT controller and adapted using LoRA with rank 16, scaling 32, and dropout 0.05. We used GRPO with five trajectories per question, a training batch size of 8, and a clipping ratio of 0.2. The correction head was optimized using AdamW with a learning rate of \(10^{-4}\), zero weight decay, and a maximum gradient norm of 1.0. The SFT-initialized controller parameters were kept fixed during GRPO training. No KL or entropy regularization was applied. The structured route hint used token budgets of 16, 12, and 12 for the ENTITY, RELATION, and DISAMBIGUATOR fields, with temperature 1.0, top-\(p=0.95\), and correction Top-\(K=128\). We trained on 2,870 questions, used 359 questions for testing.

We used Qwen/Qwen3-Embedding-8B as the fixed dense retriever. The maximum research horizon was set to 80 rounds, with an input budget of 23,000 tokens and a maximum model context length of 32,768 tokens. For state-based action decisions, we used the host model itself to extract attention-based signals and confidence scores. We set the confidence threshold to 0.95 and the entropy-delta threshold to \(-0.10\), using a 32-token decoding attention window and monitoring the final four layers. It generated an ENTITY ; RELATION ; DISAMBIGUATOR route hint with token budgets of 16, 12, and 12, respectively, using correction Top-\(K=128\). The frozen Host rewrote this hint into a browser.search query using temperature 0.7, a maximum of 32 generated tokens, and a 16-word query limit.

% \begin{table}[t]
% \centering
% \caption{Hyperparameters for supervised plug-in initialization.}
% \label{tab:sft_hyperparameters}
% \small
% \setlength{\tabcolsep}{6pt}
% \begin{tabular}{ll}
% \toprule
% \textbf{Hyperparameter} & \textbf{Value} \\
% \midrule
% Base model & Qwen3-0.6B-Base \\
% Training method & LoRA SFT \\
% Action labels & Continue, Stop, Reroute \\
% Maximum sequence length & 8,192 \\
% Training epochs & 3 \\
% Learning rate & \(2\times10^{-4}\) \\
% Optimizer & AdamW \\
% Learning-rate schedule & Cosine \\
% Warmup ratio & 0.03 \\
% LoRA rank & 16 \\
% LoRA scaling factor & 32 \\
% LoRA dropout & 0.05 \\
% Per-device batch size & 1 \\
% Gradient accumulation steps & 4 \\
% Number of GPUs & 4 \\
% Effective batch size & 16 \\
% Class-sampling exponent & 0.5 \\
% \bottomrule
% \end{tabular}
% \end{table}

\section{Turn-Level Supervision and Data Construction}
\label{trajactory_label}

\paragraph{Turn-Level Supervision.}
We collect multi-turn deep research trajectories from a pretrained deep-research agent and retain trajectories with correct final answers.  Upon completion of each deep-research turn \(t\), we extract a state \(s_t=(q,h_{<t},a_t)\), where \(q\) is the original question, \(h_{<t}\) is the  deep research history preceding the current turn, and \(a_t\) is the current action together with its result. We extract \(a_t\) from the final tool call of turn \(t\), such as \texttt{search}, \texttt{open}, or \texttt{find}.
%%%%
A capable language model, such as GPT-4.1, evaluates the current action \(a_t\) in the context of the question \(q\) and  the accumulated deep-research history \(h_{<t}\), using the gold answer as the annotation target. It first determines whether the cumulative evidence after \(a_t\) supports all essential parts of the gold answer. If so, it sets \(z_t^{\mathrm{stop}}=1\). Otherwise, it sets \(z_t^{\mathrm{stop}}=0\) and evaluates whether the search direction represented by \(a_t\) remains productive given the preceding history. It sets \(z_t^{\mathrm{reroute}}=1\) if the direction is unproductive and requires a substantial change, and \(z_t^{\mathrm{reroute}}=0\) if the direction remains productive and should be continued. A single failed tool call is insufficient to trigger rerouting. When \(z_t^{\mathrm{stop}}=1\), we set \(z_t^{\mathrm{reroute}}=0\).
%%%%%
The resulting label is:
\begin{equation}
\quad
y_t=\textsc{Reroute}\ \text{if }z_t^{\mathrm{stop}}=0\land z_t^{\mathrm{reroute}}=1;\quad
y_t=\textsc{Continue}\ \mathrm{o.w.}.
\label{eq:three_way_labels}
\end{equation}
Appendix~\ref{app:prompt} reports the trajectory-generation model and annotation prompt. Using these annotations, we initialize the plug-in from Qwen3-0.6B-Base and fine-tune it on the turn-level action labels. Given the current state \(s_t=(q,h_{<t},a_t)\), the plug-in generates exactly one label from \(\{\textsc{Continue}, \textsc{Reroute}\}\).
% We mask the input tokens and compute the language-modeling loss only on the target label:
% \begin{equation}
% \mathcal{L}*{\mathrm{SFT}}=-\log p*{\theta_P}(y_t\mid s_t).
% \end{equation}
Training details are provided in Appendix~\ref{im_details}.

% \paragraph{Trajectory generation and annotation.}
\paragraph{Training Data Construction.} We use \textbf{Tongyi-DeepResearch-30B-A3B} to generate complete multi-turn deep-research trajectories for \(3{,}229\) questions. To avoid test leakage, trajectories corresponding to the 359 BrowseCore test questions are excluded from controller training.  We then use GPT-4.1 to evaluate the correctness of each final answer and retain \(1{,}575\) correct trajectories, comprising \(32{,}351\) turns ending with actions such as \texttt{search}, \texttt{open}, or \texttt{find}. For each turn \(t\), GPT-4.1 receives the original question, the gold answer, the accumulated deep-research trajectory through the current action result, and the current action type. The gold answer is used only as the annotation target and is not treated as evidence. The judge first determines whether the accumulated evidence is sufficient to answer the question and then, when the evidence is insufficient, determines whether the current search direction remains productive. After annotation-quality filtering, \(1{,}475\) questions have complete turn-level annotations, yielding \(29{,}252\) examples. 
% We further retain \(103\) valid \textsc{Reroute} examples from partially annotated trajectories, producing a final dataset of \(29{,}252\) examples from \(1{,}503\) questions. The dataset contains \(17{,}483\) \textsc{Continue}, \(10{,}903\) \textsc{Stop}, and \(866\) \textsc{Reroute} examples.
We split the data by question ID into \(26{,}351\) training examples and \(2{,}901\) test examples, ensuring that all turns associated with the same question remain in a single partition. The complete annotation prompt is provided at Figure~\ref{fig:turn_level_annotation_prompt}. Turn-level  action annotation with GPT-4.1 cost approximately \$1{,}800.

\section{Prompt}
\label{app:prompt}

\subsection{Deep Research System Prompt}
\label{con:systempromt}

Figure~\ref{fig:literesearcher_system_prompt} presents an example system prompt
used by LiteResearcher-4B in our experiments. The prompt specifies the agent role,
tool-use protocol, reasoning format, and final answer format. It illustrates how
the Deep Research agent is instructed to iteratively reason, interact with browser
tools, and return a final answer after sufficient information has been collected.

\begin{figure*}[t]
\centering

\begin{tcolorbox}[
    enhanced,
    width=0.94\textwidth,
    colback=promptgray,
    colframe=black,
    boxrule=0.9pt,
    arc=2.8mm,
    outer arc=2.8mm,
    left=5mm,
    right=5mm,
    top=3mm,
    bottom=3mm,
    title={\textbf{System Prompt for LiteResearcher}},
    colbacktitle=black,
    coltitle=white,
    fonttitle=\small\bfseries,
    attach boxed title to top left={
        xshift=0mm,
        yshift=-0.5mm
    },
    boxed title style={
        colback=black,
        colframe=black,
        boxrule=0pt,
        arc=2.8mm,
        outer arc=2.8mm,
        left=3mm,
        right=3mm,
        top=1mm,
        bottom=1mm
    }
]

\small
\setlength{\parindent}{0pt}
\setlength{\parskip}{2pt}

You are a \textbf{Deep AI Research Assistant}.

\vspace{1mm}

The question is a complex question that requires
\textbf{deep research} to answer.

The runtime provides browser tools for searching information
and inspecting retrieved webpages.

\vspace{1mm}

You do not need to answer the question immediately.
You should first reason about the research plan or determine
what to search next.

\vspace{1mm}

\textbf{Output Format}

Your output must follow one of the following two formats.

\begin{tcolorbox}[
    colback=codegray,
    colframe=gray!45,
    boxrule=0.4pt,
    arc=1.5mm,
    left=3mm,
    right=3mm,
    top=2mm,
    bottom=2mm
]
\ttfamily\footnotesize
<think>\\
YOUR THINKING PROCESS\\
</think>\\
<answer>\\
YOUR ANSWER AFTER GETTING ENOUGH INFORMATION\\
</answer>
\end{tcolorbox}

or

\begin{tcolorbox}[
    colback=codegray,
    colframe=gray!45,
    boxrule=0.4pt,
    arc=1.5mm,
    left=3mm,
    right=3mm,
    top=2mm,
    bottom=2mm
]
\ttfamily\footnotesize
<think>\\
YOUR THINKING PROCESS\\
</think>\\
<tool\_call>\\
YOUR TOOL CALL WITH CORRECT FORMAT\\
</tool\_call>
\end{tcolorbox}

\vspace{1mm}

The runtime provides the available browser function signatures separately.
Use only the exact function names and argument schemas supplied by the runtime.
Do not invent aliases or rely on any other tool interface.

\vspace{1mm}

For each function call, return a JSON object with the function name and
arguments inside \texttt{<tool\_call></tool\_call>} tags.

\begin{tcolorbox}[
    colback=codegray,
    colframe=gray!45,
    boxrule=0.4pt,
    arc=1.5mm,
    left=3mm,
    right=3mm,
    top=2mm,
    bottom=2mm
]
\ttfamily\footnotesize
<tool\_call>\\
\{\\
\hspace*{4mm}
"name": <function-name>,\\
\hspace*{4mm}
"arguments": <args-json-object>\\
\}\\
</tool\_call>
\end{tcolorbox}

\vspace{1mm}

You should always follow the above formats strictly.

Only output the final answer in words, numbers, or a short phrase inside
the \texttt{<answer></answer>} tag, without explanations or extra information.
For yes-or-no questions, output only \texttt{yes} or \texttt{no}.

\vspace{1mm}

\textbf{Current date:} \textcolor{promptblue}{\{current\_date\}}

\end{tcolorbox}

\caption{
System prompt used by LiteResearcher during Deep Research.
The model iteratively reasons, invokes browser tools, and produces a final
answer after sufficient evidence has been collected.
}
\label{fig:literesearcher_system_prompt}

\end{figure*}

\subsection{ Browser Tools}
\label{con:browser_tools}
\begin{figure*}[t]
\centering

\begin{tcolorbox}[
    enhanced,
    width=0.94\textwidth,
    colback=promptgray,
    colframe=black,
    boxrule=0.9pt,
    arc=2.8mm,
    outer arc=2.8mm,
    left=5mm,
    right=5mm,
    top=3mm,
    bottom=3mm,
    title={\textbf{Browser Tool Specification for LiteResearcher-4B}},
    colbacktitle=black,
    coltitle=white,
    fonttitle=\small\bfseries,
    attach boxed title to top left={
        xshift=0mm,
        yshift=-0.5mm
    },
    boxed title style={
        colback=black,
        colframe=black,
        boxrule=0pt,
        arc=2.8mm,
        outer arc=2.8mm,
        left=3mm,
        right=3mm,
        top=1mm,
        bottom=1mm
    }
]

\small
\setlength{\parindent}{0pt}
\setlength{\parskip}{2pt}

The runtime provides three browser tools for external information retrieval
and webpage inspection.

\vspace{1mm}

\textbf{browser.search}

Searches for information related to a query and returns the top retrieved
results with titles, URLs, and summaries.

\begin{tcolorbox}[
    colback=codegray,
    colframe=gray!45,
    boxrule=0.4pt,
    arc=1.5mm,
    left=3mm,
    right=3mm,
    top=2mm,
    bottom=2mm
]
\ttfamily\footnotesize
\{\\
\hspace*{4mm}"query": <string>,\\
\hspace*{4mm}"topn": <integer, default=10>\\
\}
\end{tcolorbox}

\vspace{1mm}

\textbf{browser.open}

Opens a retrieved link or a fully qualified URL and optionally displays a
specified region of the page.

\begin{tcolorbox}[
    colback=codegray,
    colframe=gray!45,
    boxrule=0.4pt,
    arc=1.5mm,
    left=3mm,
    right=3mm,
    top=2mm,
    bottom=2mm
]
\ttfamily\footnotesize
\{\\
\hspace*{4mm}"id": <integer or string>,\\
\hspace*{4mm}"cursor": <integer>,\\
\hspace*{4mm}"loc": <integer>,\\
\hspace*{4mm}"num\_lines": <integer>,\\
\hspace*{4mm}"view\_source": <boolean>,\\
\hspace*{4mm}"source": <string>\\
\}
\end{tcolorbox}

\vspace{1mm}

\textbf{browser.find}

Finds exact matches of a text pattern within the current or specified webpage.

\begin{tcolorbox}[
    colback=codegray,
    colframe=gray!45,
    boxrule=0.4pt,
    arc=1.5mm,
    left=3mm,
    right=3mm,
    top=2mm,
    bottom=2mm
]
\ttfamily\footnotesize
\{\\
\hspace*{4mm}"pattern": <string>,\\
\hspace*{4mm}"cursor": <integer>\\
\}
\end{tcolorbox}

\end{tcolorbox}

\caption{
Browser tool specifications provided to LiteResearcher-4B.
The agent uses \texttt{browser.search} for retrieval,
\texttt{browser.open} for webpage inspection, and
\texttt{browser.find} for locating specific information within a page.
}
\label{fig:literesearcher_tool_specification}

\end{figure*}
Figure~\ref{fig:literesearcher_tool_specification} presents an example of
the browser tool interfaces provided to LiteResearcher-4B in our experiments.
The runtime exposes tools for web search, webpage inspection, and in-page
information lookup, enabling the agent to iteratively gather evidence during
long-horizon research.

\subsection{Turn-Level Annotation for Plug-in Training}
\label{con:turn-level prompt}

\begin{figure*}[t]
\centering

\begin{tcolorbox}[
    enhanced,
    width=0.96\textwidth,
    colback=promptgray,
    colframe=black,
    boxrule=0.9pt,
    arc=2.8mm,
    outer arc=2.8mm,
    left=5mm,
    right=5mm,
    top=3mm,
    bottom=3mm,
    title={\textbf{Prompt Template for Turn-Level Action Annotation}},
    colbacktitle=black,
    coltitle=white,
    fonttitle=\small\bfseries,
    attach boxed title to top left={
        xshift=0mm,
        yshift=-0.5mm
    },
    boxed title style={
        colback=black,
        colframe=black,
        boxrule=0pt,
        arc=2.8mm,
        outer arc=2.8mm,
        left=3mm,
        right=3mm,
        top=1mm,
        bottom=1mm
    }
]

\small
\setlength{\parindent}{0pt}
\setlength{\parskip}{2pt}

You are a strict evidence-sufficiency judge.

\vspace{1mm}

Evaluate
\textcolor{promptblue}{\{case\_count\}}
independent cases. For each case, use only the evidence in that case's
cumulative history. Never transfer information between cases and do not use
outside knowledge. A ground-truth answer is only the evaluation target; it is
not itself evidence.

\vspace{1mm}

\textbf{Judgment A: Evidence sufficiency.}

\begin{itemize}[leftmargin=4mm,itemsep=1pt,topsep=1pt]
    \item \textbf{STOP = 1:} The evidence collected so far contains
    sufficient information to support all essential claims in the
    ground-truth answer. The question can therefore be answered correctly
    without further retrieval.

    \item \textbf{STOP = 0:} At least one essential part of the
    ground-truth answer is unsupported, unclear, ambiguous, or requires
    additional retrieval.
\end{itemize}

An unsupported guess or claim made by the assistant does not count as evidence.

\vspace{1mm}

\textbf{Judgment B: Search-direction assessment.}

If \textbf{STOP = 0}, evaluate the current search direction:

\begin{itemize}[leftmargin=4mm,itemsep=1pt,topsep=1pt]
    \item \textbf{REROUTE = 1:} The current search direction is no longer
    productive. The agent should switch to a substantially different
    direction to obtain the missing evidence.

    \item \textbf{REROUTE = 0:} The current search direction remains
    productive, and further retrieval along the same direction may provide
    the missing evidence.
\end{itemize}

A single failed tool call is not sufficient to trigger a reroute when the
broader search direction remains promising.

\vspace{1mm}

\textbf{Label Definitions}

\begin{itemize}[leftmargin=4mm,itemsep=1pt,topsep=1pt]
    \item \textbf{REROUTE:}
    \(\mathrm{STOP}=0\) and \(\mathrm{REROUTE}=1\).

    \item \textbf{CONTINUE:}
    \(\mathrm{STOP}=0\) and \(\mathrm{REROUTE}=0\). The available evidence
    is insufficient, but the current search direction remains productive
    and should be continued.
\end{itemize}

If \textbf{STOP = 1}, \textbf{REROUTE} must be \(0\), because no further
retrieval is needed.

\vspace{1mm}

\textbf{Output Format}

Return only a JSON array in \texttt{case\_id} order, with no Markdown or
explanation:

\begin{tcolorbox}[
    colback=codegray,
    colframe=gray!40,
    boxrule=0.4pt,
    arc=1.5mm,
    left=3mm,
    right=3mm,
    top=2mm,
    bottom=2mm
]
\ttfamily\footnotesize
\detokenize{[{"case_id":0,"stop":0,"reroute":0}]}
\end{tcolorbox}

The array must contain exactly
\textcolor{promptblue}{\{case\_count\}}
objects, one per case.

\vspace{1mm}

\textbf{Input Cases}

\textcolor{promptblue}{\{cases\}}

\end{tcolorbox}

\caption{Prompt template used to annotate each deep-research turn with
\textsc{Continue},  or \textsc{Reroute}.}
\label{fig:turn_level_annotation_prompt}

\end{figure*}

Figure~\ref{fig:turn_level_annotation_prompt} shows the prompt template used to annotate each deep-research turn with one of two labels: \textsc{Continue}   or \textsc{Reroute}, providing turn-level supervision for training the plug-in model.

% =========================
% Preamble
% =========================

\definecolor{promptblue}{RGB}{45,95,180}
\definecolor{promptpurple}{RGB}{125,70,170}
\definecolor{promptgray}{RGB}{247,247,247}
\definecolor{codegray}{RGB}{242,242,242}

% =========================
% Prompt Figure
% =========================
\begin{figure*}[t]
\centering

\begin{tcolorbox}[
    enhanced,
    width=0.96\textwidth,
    colback=promptgray,
    colframe=black,
    boxrule=0.9pt,
    arc=2.8mm,
    outer arc=2.8mm,
    left=5mm,
    right=5mm,
    top=3mm,
    bottom=3mm,
    title={\textbf{Prompt Template for Evidence Document Verification}},
    colbacktitle=black,
    coltitle=white,
    fonttitle=\small\bfseries,
    attach boxed title to top left={
        xshift=0mm,
        yshift=-0.5mm
    },
    boxed title style={
        colback=black,
        colframe=black,
        boxrule=0pt,
        arc=2.8mm,
        outer arc=2.8mm,
        left=3mm,
        right=3mm,
        top=1mm,
        bottom=1mm,
    }
]

\small
\setlength{\parindent}{0pt}
\setlength{\parskip}{2pt}

You are labeling candidate documents for a multi-hop question answering dataset.

\vspace{1mm}

\textbf{Original question:}
\textcolor{promptblue}{\{original\_question\}}

\textbf{Ground-truth answer:}
\textcolor{promptblue}{\{original\_answer\}}

\textbf{Target clue:}
\textcolor{promptblue}{\{target\_subquestion\}}

\textbf{Candidate documents:}
\textcolor{promptblue}{\{documents\_json\}}

\vspace{1mm}

For each candidate document, make two judgments.

\vspace{1mm}

\textbf{Judgment A: Clue-level evidence assessment.}

Classify the document with respect to the target clue:

\begin{itemize}[leftmargin=4mm,itemsep=1pt,topsep=1pt]
    \item \textbf{SUPPORTS:} The document provides explicit or clearly inferable evidence that answers the target clue.
    
    \item \textbf{RELATED\_BUT\_INSUFFICIENT:} The document is relevant to the topic, entity, or clue, but does not provide sufficient evidence to answer the target clue.
    
    \item \textbf{IRRELEVANT:} The document does not meaningfully help answer the target clue.
\end{itemize}

\vspace{1mm}

\textbf{Judgment B: Original-question evidence assessment.}

Determine whether this single document contains sufficient evidence to answer the original question with the given ground-truth answer.

Set
\textcolor{promptpurple}{\texttt{can\_answer\_original=true}}
only if the document contains sufficient explicit or logically inferable evidence to derive the final answer.

A gold document is not defined solely by exact string matching. The answer may be implicit, paraphrased, or expressed in a different surface form. If the document only contains the answer string without establishing its relation to the original question, set
\textcolor{promptpurple}{\texttt{can\_answer\_original=false}}.

\vspace{1mm}

\textbf{Important Rules}

\begin{itemize}[leftmargin=4mm,itemsep=1pt,topsep=1pt]
    \item Do not use outside knowledge.
    \item Do not treat shared keywords, similar topics, or partial entity matches as sufficient evidence.
    \item If the label is \textcolor{promptpurple}{\texttt{SUPPORTS}}, identify the shortest supporting evidence span.
    \item If \textcolor{promptpurple}{\texttt{can\_answer\_original=true}}, identify the shortest span supporting the final answer.
    \item Preserve the original \texttt{document\_id} for every candidate document.
\end{itemize}

\vspace{1mm}

\textbf{Output Format}

\begin{tcolorbox}[
    colback=codegray,
    colframe=gray!40,
    boxrule=0.4pt,
    arc=1.5mm,
    left=3mm,
    right=3mm,
    top=2mm,
    bottom=2mm
]
\ttfamily\footnotesize
\{
\ "results": [
\{
\ "document\_id": "...",\\
\ "label": "SUPPORTS" | "RELATED\_BUT\_INSUFFICIENT" | "IRRELEVANT",\\
\ "evidence\_span": "...",\\
\ "rationale": "...",\\
\ "can\_answer\_original": true | false,\\
\ "answer\_evidence\_span": "...",\\
\ "answer\_rationale": "..."
\}
]
\}
\end{tcolorbox}

\end{tcolorbox}

\caption{Prompt template for automated evidence-document verification.}
\label{fig:evidence_verification_prompt}

\end{figure*}

\subsection{Evidence Document Verification}
\label{con:evidence document prompt}
Figure~\ref{fig:evidence_verification_prompt} presents the prompt template
used to verify candidate evidence documents. For each document, the annotator
first determines whether it \textsc{Supports}, is
\textsc{Related-but-Insufficient} to, or is \textsc{Irrelevant} to the
target clue. It then determines whether the document contains sufficient
evidence to answer the original question. For positive cases, the shortest
supporting evidence span is also extracted.

\subsection{Prompt Template for Clue Decomposition}
\label{clue promt}
Figure~\ref{fig:clue_decomposition_prompt} shows the prompt used for clue decomposition.
% =========================
% Preamble
% =========================
% \usepackage[most]{tcolorbox}
% \usepackage{xcolor}
% \usepackage{enumitem}

\definecolor{promptblue}{RGB}{45,95,180}
\definecolor{promptpurple}{RGB}{125,70,170}
\definecolor{promptgray}{RGB}{248,248,248}
\definecolor{codegray}{RGB}{243,243,243}

% =========================
% Clue Decomposition Prompt
% =========================
\begin{figure*}[t]
\centering

\begin{tcolorbox}[
    enhanced,
    width=0.94\textwidth,
    colback=promptgray,
    colframe=black,
    boxrule=0.9pt,
    arc=2.8mm,
    outer arc=2.8mm,
    left=5mm,
    right=5mm,
    top=3mm,
    bottom=3mm,
    title={\textbf{Prompt Template for Clue Decomposition}},
    colbacktitle=black,
    coltitle=white,
    fonttitle=\small\bfseries,
    attach boxed title to top left={
        xshift=0mm,
        yshift=-0.5mm
    },
    boxed title style={
        colback=black,
        colframe=black,
        boxrule=0pt,
        arc=2.8mm,
        outer arc=2.8mm,
        left=3mm,
        right=3mm,
        top=1mm,
        bottom=1mm
    }
]

\small
\setlength{\parindent}{0pt}
\setlength{\parskip}{2pt}

You are decomposing a complex question into
\textbf{atomic factual clues} for evidence retrieval.

\vspace{1mm}

You are given an original question and its reference answer.
Your task is to extract the individual factual conditions in the
question that the reference answer must satisfy.

\vspace{1mm}

\textbf{Original question:}
\textcolor{promptblue}{\{question\}}

\textbf{Reference answer:}
\textcolor{promptblue}{\{answer\}}

\vspace{1mm}

\textbf{Rules}

\begin{itemize}[
    leftmargin=4mm,
    itemsep=1.2pt,
    topsep=1pt,
    parsep=0pt
]
    \item Generate the \textbf{smallest complete set of clues}
    needed to cover the conditions in the original question.

    \item Each clue must express only
    \textbf{one factual condition or relation}.

    \item Write each clue as a
    \textbf{self-contained declarative statement}.

    \item Preserve important constraints such as
    \textbf{dates, locations, author roles, affiliations,
    quantities, titles, and relationships}.

    \item Use only information stated in the
    original question or reference answer.

    \item Do not invent facts, aliases, entities,
    or supporting evidence.

    \item Do not generate retrieval queries,
    explanations, URLs, or evidence spans.

    \item Return no text outside the JSON object.
\end{itemize}

\vspace{1mm}

\textbf{Output Format}

\begin{tcolorbox}[
    colback=codegray,
    colframe=gray!45,
    boxrule=0.4pt,
    arc=1.5mm,
    left=3mm,
    right=3mm,
    top=2mm,
    bottom=2mm
]
\ttfamily\footnotesize
\{
\ "clues": [\\
\hspace*{4mm}
\{\\
\hspace*{8mm}
"clue\_id": "C1",\\
\hspace*{8mm}
"clue": "A short, self-contained factual condition."\\
\hspace*{4mm}
\}\\
\ ]\\
\}
\end{tcolorbox}

\end{tcolorbox}

\caption{Prompt template for atomic clue decomposition. Given an original
question and its reference answer, the model decomposes the information
need into a minimal set of self-contained factual clues for subsequent
evidence retrieval and verification.}
\label{fig:clue_decomposition_prompt}

\end{figure*}

%% file: 7_dataset_build.tex
% !TeX root = iclr2027_conference.tex
\section{BrowseCore}
\label{con:ourdataset}
\begin{figure*}[t]
        \centering
         % \vspace{-10mm}
        \includegraphics[width=1\columnwidth]{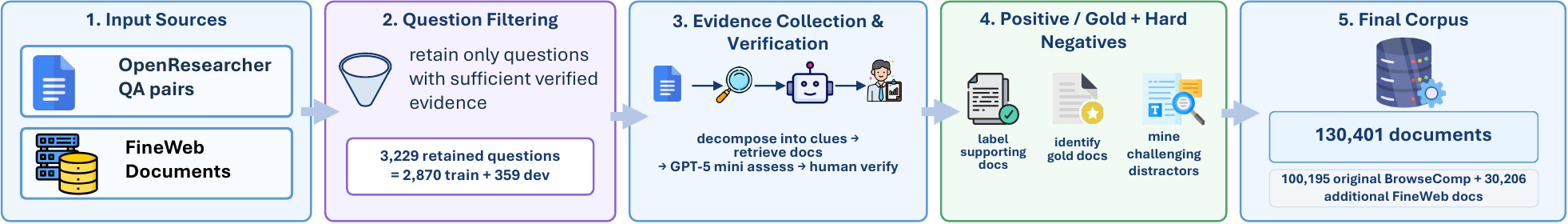}
        % \vspace{-1mm}
	\caption{Overview of our benchmark and corpus construction pipeline. We start from OpenResearcher question--answer pairs, collect candidate evidence through clue-level retrieval, and combine GPT-5 mini assessment with human verification to filter unverifiable questions. Verified supporting documents are labeled as positive and gold evidence, while high-ranking non-supporting documents are retained as hard negatives. The final corpus augments the BrowseComp-Plus retrieval corpus with these verified evidence documents and mined distractors, resulting in a compact and efficient benchmark for controlled Deep Research experiments.}
	\label{fig:dataflow}
  \vspace{-10pt}
\end{figure*}

To support controlled and reproducible retrieval experiments for Deep Research agents, we construct \textbf{BrowseCore}, a compact benchmark with manually verified evidence. BrowseCore is built from OpenResearcher question--answer pairs while substantially reducing the scale of the retrieval corpus and strengthening evidence quality through explicit verification. For each retained question, we identify supporting and gold documents, mine challenging hard negatives, and ensure that the verified evidence is sufficient to derive the reference answer. The resulting benchmark contains dedicated training and development splits and is designed to enable efficient analysis of retrieval behavior and long-horizon Deep Research trajectories.

\subsection{Background and Motivation}
\label{con:motivation}
Existing Deep Research benchmarks often rely on large-scale retrieval corpora, making controlled retrieval analysis computationally expensive. Among them, OpenResearcher provides a representative open benchmark with large-scale web documents and question--answer pairs, which we use as the starting point for constructing BrowseCore.

OpenResearcher provides approximately 10K gold documents for 6K questions in an open benchmark specifically designed to assess whether deep-research AI agents can interactively and creatively navigate the web to find complex, hard-to-find information. To approximate real-world web coverage and reflect realistic search complexity, it collects 15 million documents (approximately 10 trillion tokens) from FineWeb~\citep{penedo2024fineweb}. For efficient large-scale dense retrieval, each document is embedded with Qwen3-Embedding-8B~\citep{zhang2025qwen3} and indexed using FAISS; the resulting corpus.pkl file is 113.65 GB.

Although this setup is effective and widely used for end-to-end evaluation of integrated deep-research systems, it introduces several limitations for controlled retrieval research. During gold-document construction, OpenResearcher concatenates the question and reference answer to form the search query in order to improve recall, and then retrieves web content through the Serper API. However, the retrieved candidate documents are not exhaustively verified through human evaluation. As a result, the gold set may be incomplete or contain weakly validated evidence documents. In addition, the extremely large corpus makes experiments computationally expensive and less suitable for fine-grained analysis of retrieval behavior and research insights.
%%%%
Moreover, the retrieved documents are not systematically verified for whether they actually contain sufficient evidence to support the reference answer. Therefore, in this work, we aim to construct a high-quality and easy-to-use benchmark with a compact, carefully verified retrieval corpus.

\subsection{Benchmark and Corpus Construction}
\label{con:corpus contruction}
As shown in Figure~\ref{fig:dataflow}, we construct BrowseCore from question--answer pairs in OpenResearcher through a four-stage pipeline consisting of question filtering, evidence collection and verification, hard-negative mining, and final corpus assembly. Rather than directly treating automatically retrieved documents as gold evidence, we explicitly verify whether the retrieved documents support the decomposed clues and whether the resulting evidence is sufficient to derive the reference answer.

Firstly, we adopt a clue-level verification protocol. For each question, the question is first decomposed into self-contained clues, and candidate documents are retrieved for each clue using Qwen3-Embedding-8B. Preliminary relevance judgments and evidence spans are then generated automatically with GPT-5 mini, followed by final human verification. Annotators verify that each retained clue is supported by explicit textual evidence and that the complete set of verified evidence is sufficient for a human to derive the reference answer.
If sufficient evidence cannot be identified, the clues are revised and additional retrieval rounds are performed, for up to five rounds in total. Questions for which adequate supporting evidence still cannot be established after five rounds are removed from the benchmark. This process yields a compact and carefully verified retrieval corpus that is easier to index and retrieve from, while remaining suitable for controlled analysis of retrieval quality and long-horizon Deep Research trajectories.

\subsubsection{Question Source and Filtering}
\label{con:question source}
We begin with question--answer pairs from the training set of OpenResearcher. For each question, we first attempt to collect sufficient supporting evidence through automated retrieval and human verification. Questions are retained only when the retrieved evidence can be verified to support the key clues and is sufficient for a human to derive the reference answer. If adequate evidence cannot be identified after iterative retrieval and manual review, the question is removed from the benchmark. This filtering process results in 3,588 retained questions, including 2,870 training, 359 validation questions and 359 test questions.

\subsubsection{Candidate Document Sources}
\label{con:document source}
To collect evidence, we search over the document corpus associated with OpenResearcher, which contains approximately 15 million FineWeb documents, together with the original retrieval corpus used for the BrowseComp questions. These documents serve as the candidate pool from which supporting and distracting documents are identified.

\subsubsection{Evidence Collection and Verification}
\label{con:Evidence source}

Following BrowseComp-Plus, we first decompose each question into a set of self-contained clues, using the prompt shown in Figure~\ref{fig:clue_decomposition_prompt}, that jointly lead to the reference answer. For each clue, we use Qwen3-Embedding-8B to retrieve the top five candidate documents from the corpus.
%%%
We then use GPT-5 mini to automatically assess whether each retrieved document supports the corresponding clue, using the prompt shown in Figure~\ref{fig:evidence_verification_prompt}. The model distinguishes supporting evidence from documents that are merely related or irrelevant, and extracts the shortest supporting evidence span whenever applicable.
%%%%
Human annotators subsequently review the automatically generated relevance labels and evidence spans. For each question, annotators are required to verify that every clue is justified by explicit textual evidence and that the complete set of verified clues is sufficient for a human to derive the reference answer.

Specifically, annotators are asked to:

\begin{itemize}
\item confirm that each clue is supported by the retrieved documents and label the exact text spans that justify the clue; and
\item determine whether the combination of verified clues and supporting evidence is sufficient to answer the entire question correctly.
\end{itemize}

If the initial retrieval results do not provide sufficient evidence, annotators revise the decomposed clues and perform additional rounds of retrieval. This process is repeated for up to five rounds. If sufficient supporting evidence still cannot be identified after five retrieval rounds, the question is considered unverifiable and removed from the benchmark.

\subsubsection{Positive and Gold Document Identification}
\label{con:pos document}
Based on the evidence verification results, documents that provide sufficient support for at least one decomposed clue are labeled as positive evidence documents. Among these positives, annotators further identify documents that contain sufficient information to directly derive or verify the final reference answer; we designate these as gold documents.

Following BrowseComp-Plus, gold documents are not restricted to documents containing the reference answer as an exact substring. A document may still be considered gold when the answer is expressed implicitly, paraphrased, or in a different surface form, as long as the document provides sufficient evidence to derive the final answer. This distinction enables us to separately evaluate retrieval of useful intermediate evidence and retrieval of documents that directly support the final answer.

\subsubsection{Hard-Negative Mining}
\label{con:hard negative}
To maintain a manageable corpus while preserving retrieval difficulty, we further mine hard-negative documents as realistic distractors. Following BrowseComp-Plus, we use GPT-5 mini to decompose each original question into simpler, self-contained sub-queries. For each sub-query, Qwen3-Embedding-8B retrieves the top 10 candidate documents from the corpus.

Because these candidates are retrieved using individual clues rather than the complete set of constraints in the original question, they are often semantically relevant to part of the information need but insufficient for answering the full question. We therefore remove all documents previously verified as positive or gold evidence and retain the remaining high-ranking, non-supporting documents as hard negatives. These documents provide challenging distractors that are topically related to the query while lacking sufficient evidence to derive the reference answer.

% \subsubsection{Positive and Gold Document Identification}

% Documents verified to support one or more clues are labeled as positive evidence documents. Among these documents, annotators additionally identify those that explicitly contain or directly support the final reference answer. We designate these documents as gold documents.
% %%%%%
% This distinction allows us to separately evaluate whether a retrieval system identifies generally useful intermediate evidence and whether it retrieves documents that directly support the final answer.

% \subsubsection{Hard-Negative Mining}

% To preserve retrieval difficulty, we further mine hard-negative documents for each question. We use the decomposed clues and derived sub-queries to search the corpus for documents that are semantically similar to the information need but do not provide sufficient evidence for answering the question.
% %%%%%%
% Candidate negatives are filtered to exclude verified positive and gold documents. The remaining high-ranking but non-supporting documents are retained as hard negatives, providing challenging distractors for retrieval models.

\subsubsection{Final Corpus Assembly}
\label{con:final}
For each question, we augment the original BrowseComp-Plus retrieval corpus with our newly verified positive and gold evidence documents, together with the mined hard-negative documents. This construction preserves the retrieval difficulty of BrowseComp-Plus while ensuring that each retained question has a manually verified evidence path to the reference answer.
The resulting benchmark therefore provides both reliable evidence supervision and a compact retrieval setting suitable for controlled analysis of retrieval quality, rerouting behavior, and long-horizon Deep Research trajectories. 
The final retrieval corpus contains 130,401 documents in total, including 100,195 documents from the original BrowseComp corpus and 30,206 additional documents collected and verified from OpenResearcher. The corpus occupies approximately 1.8 GB on disk, providing a substantially more compact retrieval environment for controlled Deep Research experiments. Note: We use GPT-5 mini for clue decomposition and preliminary evidence assessment, followed by final human verification conducted by four Ph.D. student annotators.

%% file: 10_addtinial_relatedwork.tex
\section{Additional Related Work}
\label{sec:additional-related-work}
\paragraph{Retrieval-Augmented Reasoning.}
Retrieval-augmented generation (RAG) combines parametric language models with
external non-parametric memory to improve knowledge-intensive generation
\citep{lewis2020retrieval}. Beyond conventional document retrieval,
retrieval augmentation has been extended to structured knowledge, where
relevant entities, attributes, and relational structures are retrieved from
knowledge graphs to support state tracking and link prediction
\citep{li2023understand,li2024zero}. Subsequent work has increasingly coupled
retrieval with the evolving reasoning process. IRCoT
\citep{trivedi2023interleaving} interleaves retrieval with chain-of-thought
reasoning for multi-step question answering, while RT \citep{li2024rt}
combines retrieval-based demonstration selection with chain-of-thought
reasoning for few-shot biomedical named entity recognition. ReAct
\citep{yao2023react} further establishes the effectiveness of interleaving
reasoning with information-seeking actions, and DeepRAG
\citep{guan2026deeprag} formulates retrieval-augmented reasoning as a
sequential decision process that dynamically determines whether retrieval is
needed during reasoning.
%%%
Adaptive retrieval has also been explored through model-generated control
signals and dynamically constructed reasoning structures. Self-RAG
\citep{asai2024self} learns reflection tokens to determine when retrieval is
needed and to assess retrieved evidence, while LogicRAG
\citep{kabir2025logic} constructs query-specific reasoning structures to
guide multi-step retrieval. SPARKLE
\citep{fang2026sparkle} introduces a lightweight plug-and-play retrieval
controller that adaptively regulates retrieval while remaining decoupled from
the underlying retriever and language model. In domain-specific settings, BiomedRAG
\citep{li2025biomedrag} augments language models with retrieved biomedical
evidence, and GLEN-Bench \citep{huang2026glen} studies graph-language reasoning
for nutritional health, including retrieval over structured graph knowledge.
CancerLLM \citep{li2026cancerllm} further employs retrieval-based example
selection for cancer phenotyping and diagnosis generation. Complementing these
methodological developments, retrieval-augmented LLMs have also been
systematically evaluated for robustness and self-awareness under noisy and
heterogeneous evidence \citep{li2025benchmarking}.
\paragraph{Adaptive Control for Retrieval and Reasoning.}
As retrieval becomes increasingly iterative, recent work has shifted from
static retrieval pipelines toward adaptive and learned control over retrieval
and evidence use. Agentic-R \citep{liu2026agentic} learns retrievers for
multi-turn agentic search by jointly considering local retrieval relevance and
downstream answer correctness, while REPAIR \citep{kim2026adaptive} uses
reasoning plans to support mid-course adaptation during iterative retrieval.
RAG-GD \citep{li2026context} formulates retrieval-augmented generation as an
in-context optimization process and learns lightweight, context-conditioned
updates to how retrieved evidence is used, while keeping both the retriever and
language-model backbone frozen. RICE-PO \citep{li2026rice} further addresses
credit assignment in interactive retrieval agents by transforming intermediate
retrieval interactions into localized learning signals for policy optimization.
Related work has also explored adaptive control beyond retrieval itself.
EvolveRouter \citep{huang2026evolverouter} jointly adapts agent routing and
prompting for multi-agent question answering, while Counterfactual Graph
Calibration \citep{huang2026counterfactual} uses calibrated confidence to
improve communication and topology selection among multiple LLM agents.
These methods demonstrate the broader value of learned control over retrieval,
agent selection, and communication structure.
Long-horizon Deep Research introduces a further trajectory-level challenge:
a sequence of individually reasonable retrieval and reasoning decisions can
still accumulate into an inefficient or unproductive research trajectory.
\textsc{RIVO} targets this problem by learning when trajectory correction is
needed and, once triggered, applying a lightweight logit intervention to the
frozen Host model to redirect subsequent reasoning and retrieval toward a more
promising research direction.

%% file: 7_proof.tex
% !TeX root = iclr2027_conference.tex
\section{Proof of Main  Theorem}
\label{proof}
% \section{Theoretical Analysis}
We establish the theoretical properties of the RIVO logit intervention,
showing how token-level corrections can shift the generation distribution
toward higher-value reroute queries. By modeling the logit modification as
a function of the Plugin parameters $\theta$, we analyze how
trajectory-level reward learning induces changes in reroute-query
probabilities and, under suitable conditions, improves the expected
downstream value of the generated queries.

Consider a Deep Research turn $t$ at which \textsc{Reroute} is triggered.
Let $\mathcal{F}_{t-1}$ denote the research history available before
generating the reroute query at turn $t$, and let
$\mathcal{F}_{t,s-1}$ denote the corresponding history before generating
token position $s$ of that query. Let
$\theta\in\Theta\subseteq\mathbb{R}^d$ denote the trainable Plugin
parameters, and let $\mathbf{z}_{t,s}^{H}$ denote the vocabulary-level
logits produced by the frozen Host at token position $s$.
For each candidate token $x_j\in\mathcal{V}$, the Plugin produces an
effective logit modification $g_{t,s,j}(\theta)$, which modifies the Host
distribution used to generate the reroute query. Before presenting the
main theoretical result, we first introduce the assumptions required for
the analysis.
% \begin{assumption}[Smooth Logit Intervention]
% \label{auusmption1}
% Let $\theta$ denote the trainable Plugin parameters, and let
% $x_j \in \mathcal{V}$ denote the $j$-th candidate vocabulary token.
% Let $\Delta z_{t,j}(\theta;\mathbf{z}_t^H,\mathcal{F}_{t-1})$ denote
% the Plugin-induced logit modification for $x_j$ at generation step $t$,
% where $\mathbf{z}_t^H$ denotes the frozen host logits and
% $\mathcal{F}_{t-1}$ denotes the reroute-query history. Let
% $g_t(\theta;\mathbf{z}_t^H,x_j,\mathcal{F}_{t-1})$ be the Plugin
% function that determines $\Delta z_{t,j}$.
% %%%%
% We assume that, for some constants $\lambda_0,\lambda_1>0$, $\left|\nabla_\theta g_t\right| \le \lambda_0$ and $\left|\nabla_\theta^2 g_t\right| \le \lambda_1$ for all generation steps $t$ and candidate tokens $x_j$.
% %%
% Assumption 1 assumes that the history-dependent logit intervention has
% bounded gradient and Hessian, similar to standard smoothness assumptions
% used in prior analyses of deep models
% \cite{zhang2024transformers,hiranandani2025logits}.
% \end{assumption}

\begin{assumption}[Bounded Logit Intervention]
\label{assumption:bounded_intervention}

Let $g_{t,s,j}(\theta)$ denote the effective Plugin-induced logit
correction applied to candidate token $x_j\in\mathcal{V}$ at token
position $s$ of the reroute query generated at Deep Research turn $t$.
We assume that there exists a constant $G>0$ such that
$|g_{t,s,j}(\theta)|\leq G$
for all relevant Deep Research turns $t$, token positions $s$, and
candidate tokens $x_j\in\mathcal{V}$.
%%%
We further assume that the reroute-query length is bounded by
$L_t\leq L_{\max}$ and that the downstream value is bounded as
$|Q_t(q_t)|\leq Q_{\max}$.
\end{assumption}

\begin{assumption}[Value Alignment with Downstream Return]
\label{assumption:value_alignment}

Let
$q_t=(x_{t,1},\ldots,x_{t,L_t})$
denote the complete reroute query generated at Deep Research turn $t$.
We define its downstream value as
\begin{equation}
Q_t(q_t)
=
\mathbb{E}
\left[
R(\tau_{1:T})
\mid
q_t,\mathcal{F}_{t-1}
\right],
\end{equation}
where $\tau_{1:T}$ denotes the complete Deep Research trajectory.
We further define the first-order query-level probability shift as
\begin{equation}
S_t(\theta;q_t)
=
\left.
\frac{\partial}{\partial\beta}
\log
\pi_{\beta,\theta}
(q_t\mid\mathcal{F}_{t-1})
\right|_{\beta=0}.
\end{equation}

We assume that the chosen intervention direction is locally aligned with
downstream value, such that
\begin{equation}
\beta\,
\operatorname{Cov}_{q_t\sim
\pi_H(\cdot\mid\mathcal{F}_{t-1})}
\left(
Q_t(q_t),
S_t(\theta;q_t)
\right)
>0.
\end{equation}
\end{assumption}

\begin{lemma}[Query-Distribution Shift under Logit Intervention]
\label{lemma:query_distribution_shift}

Consider a reroute query
$q_t=(x_{t,1},\ldots,x_{t,L_t})$
generated at Deep Research turn $t$ when rerouting is triggered.
At token position $s$, let $\mathbf{z}_{t,s}^{H}$ denote the logits
produced by the frozen Host model, and let
$g_{t,s,j}(\theta)$ denote the Plugin-induced logit modification for
candidate token $x_j\in\mathcal{V}$.
For an intervention strength $\beta\in\mathbb{R}$, define the induced
token policy as
\begin{equation}
\pi_{\beta,\theta}
(x_j\mid\mathcal{F}_{t,s-1})
=
\frac{
\exp\!\left(
z_{t,s,j}^{H}
+
\beta g_{t,s,j}(\theta)
\right)
}{
\sum_{k\in\mathcal{V}}
\exp\!\left(
z_{t,s,k}^{H}
+
\beta g_{t,s,k}(\theta)
\right)
}.
\end{equation}

Then $\pi_{\beta,\theta}$ is differentiable with respect to $\beta$,
and
\begin{equation}
\frac{\partial}{\partial\beta}
\log
\pi_{\beta,\theta}
(x_j\mid\mathcal{F}_{t,s-1})
=
g_{t,s,j}(\theta)
-
\mathbb{E}_{x_k\sim\pi_{\beta,\theta}}
\left[
g_{t,s,k}(\theta)
\mid
\mathcal{F}_{t,s-1}
\right].
\end{equation}

Equivalently,
\begin{equation}
\frac{\partial}{\partial\beta}
\pi_{\beta,\theta}
(x_j\mid\mathcal{F}_{t,s-1})
=
\pi_{\beta,\theta}
(x_j\mid\mathcal{F}_{t,s-1})
\left(
g_{t,s,j}(\theta)
-
\mathbb{E}_{x_k\sim\pi_{\beta,\theta}}
[g_{t,s,k}(\theta)]
\right).
\end{equation}

Consequently, the probability of the complete reroute query satisfies
\begin{equation}
\pi_{\beta,\theta}
(q_t\mid\mathcal{F}_{t-1})
=
\prod_{s=1}^{L_t}
\pi_{\beta,\theta}
(x_{t,s}\mid\mathcal{F}_{t,s-1}),
\end{equation}
and its first-order probability shift is characterized by
\begin{equation}
S_t(\theta;q_t)
:=
\left.
\frac{\partial}{\partial\beta}
\log
\pi_{\beta,\theta}
(q_t\mid\mathcal{F}_{t-1})
\right|_{\beta=0}.
\end{equation}
\end{lemma}

%%%%Proof1  is begaining=====%%%%%%%

%%%%%%%%%%%
\begin{proof}

\emph{Token-level probability shift.}
For a candidate token $x_j\in\mathcal V$, the intervened token policy is
\begin{equation}
\pi_{\beta,\theta}(x_j\mid\mathcal F_{t,s-1})
=
\frac{
\exp\!\left(
z^H_{t,s,j}+\beta g_{t,s,j}(\theta)
\right)
}{
\sum_{k\in\mathcal V}
\exp\!\left(
z^H_{t,s,k}+\beta g_{t,s,k}(\theta)
\right)
}.
\end{equation}

Taking logarithms gives
\begin{equation}
\log \pi_{\beta,\theta}(x_j\mid\mathcal F_{t,s-1})
=
z^H_{t,s,j}
+
\beta g_{t,s,j}(\theta)
-
\log
\sum_{k\in\mathcal V}
\exp\!\left(
z^H_{t,s,k}
+
\beta g_{t,s,k}(\theta)
\right).
\end{equation}

Differentiating with respect to $\beta$ yields
\begin{equation}
\frac{\partial}{\partial\beta}
\log \pi_{\beta,\theta}(x_j\mid\mathcal F_{t,s-1})
=
g_{t,s,j}(\theta)
-
\sum_{k\in\mathcal V}
\pi_{\beta,\theta}(x_k\mid\mathcal F_{t,s-1})
g_{t,s,k}(\theta).
\end{equation}

This derivative characterizes the local sensitivity of the token
probability to the intervention strength $\beta$, indicating the
direction and rate of probability change with respect to $\beta$.
The second term is the probability-weighted expectation of the
Plugin-induced logit modification:
\begin{equation}
\sum_{k\in\mathcal V}
\pi_{\beta,\theta}(x_k\mid\mathcal F_{t,s-1})
g_{t,s,k}(\theta)
=
\mathbb E_{x_k\sim\pi_{\beta,\theta}}
\left[
g_{t,s,k}(\theta)
\mid
\mathcal F_{t,s-1}
\right].
\end{equation}

Hence,
\begin{equation}
\frac{\partial}{\partial\beta}
\log \pi_{\beta,\theta}(x_j\mid\mathcal F_{t,s-1})
=
g_{t,s,j}(\theta)
-
\mathbb E_{x_k\sim\pi_{\beta,\theta}}
\left[
g_{t,s,k}(\theta)
\mid
\mathcal F_{t,s-1}
\right].
\end{equation}

Using
\begin{equation}
\frac{\partial}{\partial\beta}\pi_{\beta,\theta}
=
\pi_{\beta,\theta}
\frac{\partial}{\partial\beta}\log\pi_{\beta,\theta},
\end{equation}
we obtain
\begin{equation}
\frac{\partial}{\partial\beta}
\pi_{\beta,\theta}(x_j\mid\mathcal F_{t,s-1})
=
\pi_{\beta,\theta}(x_j\mid\mathcal F_{t,s-1})
\left(
g_{t,s,j}(\theta)
-
\mathbb E_{x_k\sim\pi_{\beta,\theta}}
\left[
g_{t,s,k}(\theta)
\mid
\mathcal F_{t,s-1}
\right]
\right).
\end{equation}

This identity gives a direct interpretation of the probability shift.
For brevity, let $\mathbb{E}[g]$ denote the probability-weighted average
Plugin-induced logit correction over the vocabulary.
Then $g_{t,s,j}(\theta)-\mathbb{E}[g]$ represents the relative correction
assigned to token $x_j$ with respect to the current token distribution.
Consequently, the change in its probability is determined by both its current
probability $\pi_{\beta,\theta}(x_j)$ and this relative correction.
In particular, $g_{t,s,j}(\theta)>\mathbb{E}[g]$ implies
$\partial\pi_{\beta,\theta}(x_j)/\partial\beta>0$, whereas
$g_{t,s,j}(\theta)<\mathbb{E}[g]$ implies
$\partial\pi_{\beta,\theta}(x_j)/\partial\beta<0$.
Thus, a positive change in $\beta$ shifts probability mass toward tokens
receiving above-average logit corrections, whereas a negative change in
$\beta$ reverses this direction.

\emph{Query-level probability shift.}
Let
\begin{equation}
q_t=(x_{t,1},x_{t,2},\ldots,x_{t,L_t})
\end{equation}
denote a complete reroute query. By autoregressive factorization,
\begin{equation}
\pi_{\beta,\theta}(q_t\mid\mathcal F_{t-1})
=
\prod_{s=1}^{L_t}
\pi_{\beta,\theta}
(x_{t,s}\mid\mathcal F_{t,s-1}).
\end{equation}

Taking logarithms,
\begin{equation}
\log
\pi_{\beta,\theta}(q_t\mid\mathcal F_{t-1})
=
\sum_{s=1}^{L_t}
\log
\pi_{\beta,\theta}
(x_{t,s}\mid\mathcal F_{t,s-1}).
\end{equation}

Differentiating with respect to $\beta$ gives
\begin{equation}
\frac{\partial}{\partial\beta}
\log
\pi_{\beta,\theta}(q_t\mid\mathcal F_{t-1})
=
\sum_{s=1}^{L_t}
\frac{\partial}{\partial\beta}
\log
\pi_{\beta,\theta}
(x_{t,s}\mid\mathcal F_{t,s-1}).
\end{equation}

Applying the token-level result yields
\begin{equation}
\frac{\partial}{\partial\beta}
\log
\pi_{\beta,\theta}(q_t\mid\mathcal F_{t-1})
=
\sum_{s=1}^{L_t}
\left[
g_{t,s,x_{t,s}}(\theta)
-
\mathbb E_{x_k\sim\pi_{\beta,\theta}}
\left[
g_{t,s,k}(\theta)
\mid
\mathcal F_{t,s-1}
\right]
\right].
\end{equation}

\emph{First-order shift around the host policy.}
At $\beta=0$, the intervention vanishes and
$\pi_{0,\theta}=\pi_H$. Therefore,
\begin{equation}
\left.
\frac{\partial}{\partial\beta}
\log
\pi_{\beta,\theta}(q_t\mid\mathcal F_{t-1})
\right|_{\beta=0}
=
\sum_{s=1}^{L_t}
\left[
g_{t,s,x_{t,s}}(\theta)
-
\mathbb E_{x_k\sim\pi_H}
\left[
g_{t,s,k}(\theta)
\mid
\mathcal F_{t,s-1}
\right]
\right].
\end{equation}

We define the resulting first-order query-level shift as
\begin{equation}
S_t(\theta;q_t)
:=
\left.
\frac{\partial}{\partial\beta}
\log
\pi_{\beta,\theta}(q_t\mid\mathcal F_{t-1})
\right|_{\beta=0}.
\end{equation}

This establishes how the token-level logit intervention induces a
first-order shift in the probability of a complete reroute query around
the frozen host policy.

%%%%%%%  lemma 1 end%%%%%%%%%%%%%%
%%%%%%%%%%%%%%%%%%%
\end{proof}
\begin{lemma}[First-Order Downstream Return Characterization]
\label{lemma2:First-Order Downstream Return Characterization}
Fix the current history $\mathcal{F}_{t-1}$ and the downstream
continuation policy. Define
\begin{equation}
J_t(\beta)
:=
\mathbb{E}_{q_t\sim
\pi_{\beta,\theta}(\cdot\mid\mathcal{F}_{t-1})}
\left[
Q_t(q_t)
\right],
\end{equation}
where
\begin{equation}
Q_t(q_t)
:=
\mathbb{E}
\left[
R(\tau_{1:T})
\mid
\mathcal{F}_{t-1},q_t
\right].
\end{equation}
Here,  $J_t(\beta)$ denotes the expected terminal task reward when a logit intervention of strength $\beta$ is applied to the reroute query at Deep Research turn $t$, conditioned on the current history $\mathcal{F}_{t-1}$, and the agent subsequently completes the remaining Deep Research trajectory. $R(\tau_{1:T})$ denotes the terminal task reward evaluated after
the complete Deep Research trajectory terminates. Thus, $Q_t(q_t)$
measures the expected final task reward obtained after selecting
reroute query $q_t$ and completing the remaining retrieval and
reasoning process under the fixed downstream continuation policy.

Then,
For a general intervention strength $\beta$, define the query-level
score as
\begin{equation}
S_{t,\beta}(q_t)
:=
\frac{\partial}{\partial\beta}
\log
\pi_{\beta,\theta}
(q_t\mid\mathcal{F}_{t-1}).
\end{equation}
Then,
\begin{equation}
\frac{\partial J_t(\beta)}{\partial\beta}
=
\mathbb{E}_{q_t\sim
\pi_{\beta,\theta}(\cdot\mid\mathcal{F}_{t-1})}
\left[
Q_t(q_t)S_{t,\beta}(q_t)
\right].
\label{eq:general_first_order_return}
\end{equation}

At $\beta=0$, we have
$S_{t,0}(q_t)=S_t(\theta;q_t)$ and
$\pi_{0,\theta}=\pi_H$. Therefore,
\begin{equation}
\left.
\frac{\partial J_t(\beta)}{\partial\beta}
\right|_{\beta=0}
=
\mathbb{E}_{q_t\sim
\pi_H(\cdot\mid\mathcal{F}_{t-1})}
\left[
Q_t(q_t)S_t(\theta;q_t)
\right].
\end{equation}

Moreover,
\begin{equation}
\mathbb{E}_{q_t\sim
\pi_H(\cdot\mid\mathcal{F}_{t-1})}
\left[
S_t(\theta;q_t)
\right]
=0.
\end{equation}

Therefore,
\begin{equation}
\boxed{
\left.
\frac{\partial J_t(\beta)}{\partial\beta}
\right|_{\beta=0}
=
\operatorname{Cov}_{q_t\sim
\pi_H(\cdot\mid\mathcal{F}_{t-1})}
\left(
Q_t(q_t),
S_t(\theta;q_t)
\right)
}.
\end{equation}
\end{lemma}

\begin{proof}
Since the current history $\mathcal{F}_{t-1}$ is fixed, expanding the
expectation over all possible reroute queries gives
\begin{equation}
J_t(\beta)
=
\sum_{q_t}
\pi_{\beta,\theta}
\left(
q_t\mid\mathcal{F}_{t-1}
\right)
Q_t(q_t).
\end{equation}
Under Assumption~\ref{assumption:bounded_intervention}, the relevant
quantities are bounded, allowing differentiation to be interchanged with
the expectation over reroute queries.
Because the downstream continuation policy is fixed, $Q_t(q_t)$ does
not depend on $\beta$. Differentiating with respect to $\beta$ therefore
gives
\begin{equation}
\frac{\partial J_t(\beta)}{\partial\beta}
=
\sum_{q_t}
Q_t(q_t)
\frac{\partial}{\partial\beta}
\pi_{\beta,\theta}
\left(
q_t\mid\mathcal{F}_{t-1}
\right).
\end{equation}

Using the identity
\begin{equation}
\frac{\partial}{\partial\beta}
\pi_{\beta,\theta}
\left(
q_t\mid\mathcal{F}_{t-1}
\right)
=
\pi_{\beta,\theta}
\left(
q_t\mid\mathcal{F}_{t-1}
\right)
\frac{\partial}{\partial\beta}
\log
\pi_{\beta,\theta}
\left(
q_t\mid\mathcal{F}_{t-1}
\right),
\end{equation}
we obtain
\begin{equation}
\frac{\partial J_t(\beta)}{\partial\beta}
=
\mathbb{E}_{q_t\sim
\pi_{\beta,\theta}(\cdot\mid\mathcal{F}_{t-1})}
\left[
Q_t(q_t)
\frac{\partial}{\partial\beta}
\log
\pi_{\beta,\theta}
\left(
q_t\mid\mathcal{F}_{t-1}
\right)
\right].
\end{equation}
The expectation above is a probability-weighted average over all possible
reroute queries, where each query $q_t$ is weighted by its current
probability $\pi_{\beta,\theta}(q_t\mid\mathcal F_{t-1})$.
Evaluating at $\beta=0$, where
$\pi_{0,\theta}=\pi_H$, and using the definition
of $S_t(\theta;q_t)$ from Lemma~\ref{lemma:query_distribution_shift}, yields
\begin{equation}
\left.
\frac{\partial J_t(\beta)}{\partial\beta}
\right|_{\beta=0}
=
\mathbb{E}_{q_t\sim
\pi_H(\cdot\mid\mathcal{F}_{t-1})}
\left[
Q_t(q_t)S_t(\theta;q_t)
\right].
\end{equation}

Next, the query-level score has zero expectation under the host policy:
\begin{align}
\mathbb{E}_{q_t\sim
\pi_H(\cdot\mid\mathcal{F}_{t-1})}
\left[
S_t(\theta;q_t)
\right]
&=
\sum_{q_t}
\pi_H
\left(
q_t\mid\mathcal{F}_{t-1}
\right)
\left.
\frac{\partial}{\partial\beta}
\log
\pi_{\beta,\theta}
\left(
q_t\mid\mathcal{F}_{t-1}
\right)
\right|_{\beta=0}
\\
&=
\left.
\sum_{q_t}
\frac{\partial}{\partial\beta}
\pi_{\beta,\theta}
\left(
q_t\mid\mathcal{F}_{t-1}
\right)
\right|_{\beta=0}
\\
&=
\left.
\frac{\partial}{\partial\beta}
\sum_{q_t}
\pi_{\beta,\theta}
\left(
q_t\mid\mathcal{F}_{t-1}
\right)
\right|_{\beta=0}
\\
&=
\frac{\partial}{\partial\beta}1
=
0.
\end{align}

Therefore, by the definition of covariance,
\begin{align}
&\operatorname{Cov}_{q_t\sim
\pi_H(\cdot\mid\mathcal{F}_{t-1})}
\left(
Q_t(q_t),
S_t(\theta;q_t)
\right)
\nonumber\\
&\quad=
\mathbb{E}_{q_t\sim
\pi_H(\cdot\mid\mathcal{F}_{t-1})}
\left[
Q_t(q_t)S_t(\theta;q_t)
\right]
\nonumber\\
&\qquad-
\mathbb{E}_{q_t\sim
\pi_H(\cdot\mid\mathcal{F}_{t-1})}
\left[
Q_t(q_t)
\right]
\mathbb{E}_{q_t\sim
\pi_H(\cdot\mid\mathcal{F}_{t-1})}
\left[
S_t(\theta;q_t)
\right].
\end{align}

Since
\begin{equation}
\mathbb{E}_{q_t\sim
\pi_H(\cdot\mid\mathcal{F}_{t-1})}
\left[
S_t(\theta;q_t)
\right]
=0,
\end{equation}
it follows that
\begin{equation}
\operatorname{Cov}_{q_t\sim
\pi_H(\cdot\mid\mathcal{F}_{t-1})}
\left(
Q_t(q_t),
S_t(\theta;q_t)
\right)
=
\mathbb{E}_{q_t\sim
\pi_H(\cdot\mid\mathcal{F}_{t-1})}
\left[
Q_t(q_t)S_t(\theta;q_t)
\right].
\end{equation}

Combining the above identities gives
\begin{equation}
\boxed{
\left.
\frac{\partial J_t(\beta)}{\partial\beta}
\right|_{\beta=0}
=
\operatorname{Cov}_{q_t\sim
\pi_H(\cdot\mid\mathcal{F}_{t-1})}
\left(
Q_t(q_t),
S_t(\theta;q_t)
\right)
}.
\end{equation}
The left-hand side, $\left.\partial J_t(\beta)/\partial\beta\right|_{\beta=0}$, measures the first-order change in the expected final task performance when the logit intervention is infinitesimally activated from the host policy. The right-hand side, $\operatorname{Cov}(Q_t(q_t),S_t(\theta;q_t))$, measures the alignment between the final value of a reroute query and the direction in which the Plugin shifts its probability. Here, $S_t(\theta;q_t)$ denotes the first-order shift score of the complete reroute query $q_t$, defined as the derivative of its log-probability with respect to the intervention strength $\beta$ at $\beta=0$. Intuitively, $S_t(\theta;q_t)>0$ means that activating the Plugin tends to increase the probability of generating $q_t$, whereas $S_t(\theta;q_t)<0$ means that the query is relatively suppressed. Its magnitude reflects how strongly the probability of $q_t$ responds to the logit intervention.   The left-hand side,
$\left.\partial J_t(\beta)/\partial\beta\right|_{\beta=0}$,
measures the first-order change in expected downstream performance as the
intervention strength moves away from the frozen Host policy.
The covariance on the right-hand side measures the alignment between the
downstream value of a reroute query and its first-order probability shift.
Specifically, $S_t(\theta;q_t)>0$ means that the probability of $q_t$
increases as $\beta$ increases from zero, whereas
$S_t(\theta;q_t)<0$ means that it decreases.
Therefore, the sign of
$\operatorname{Cov}(Q_t,S_t)$ determines the locally improving direction
of the intervention. A positive covariance favors increasing $\beta$,
whereas a negative covariance favors decreasing $\beta$.
Consequently, an intervention direction satisfying
$\beta\,\operatorname{Cov}(Q_t,S_t)>0$
produces a positive first-order change in expected downstream value.
This completes the proof.
\end{proof}

%%%%%%%%%Leamma 3%%%%%%%%
\begin{lemma}[Bounded Second-Order Variation of Downstream Return]
\label{lemma3}
Assume that the Plugin-induced logit modification is uniformly bounded,
\begin{equation}
\left|g_{t,s,j}(\theta)\right| \le G,
\end{equation}
for all token positions $s$, candidate tokens $x_j$, and relevant
Plugin parameters $\theta$. Assume further that the reroute-query length
is bounded by $L_t\le L_{\max}$ and that
\begin{equation}
|Q_t(q_t)|\le Q_{\max}.
\end{equation}
Then there exists a finite constant $B_J>0$ such that
\begin{equation}
\left|
\frac{\partial^2 J_t(\beta)}{\partial\beta^2}
\right|
\le B_J
\end{equation}
for $\beta$ in a neighborhood of zero. In particular, one may take
\begin{equation}
B_J
=
Q_{\max}G^2
\left(
4L_{\max}^2+L_{\max}
\right).
\end{equation}
\end{lemma}

\begin{proof}

\emph{Bounding the query-level score.}
Define the query-level score at intervention strength $\beta$ as
\begin{equation}
S_{t,\beta}(q_t)
:=
\frac{\partial}{\partial\beta}
\log
\pi_{\beta,\theta}(q_t\mid\mathcal F_{t-1}).
\end{equation}

From Lemma~\ref{lemma:query_distribution_shift},
\begin{equation}
S_{t,\beta}(q_t)
=
\sum_{s=1}^{L_t}
\left[
g_{t,s,x_{t,s}}(\theta)
-
\mathbb E_{x_k\sim\pi_{\beta,\theta}}
\left[
g_{t,s,k}(\theta)
\mid
\mathcal F_{t,s-1}
\right]
\right].
\end{equation}

Since $|g_{t,s,j}(\theta)|\le G$, we have
\begin{equation}
\left|
g_{t,s,x_{t,s}}(\theta)
-
\mathbb E_{x_k\sim\pi_{\beta,\theta}}
\left[
g_{t,s,k}(\theta)
\mid
\mathcal F_{t,s-1}
\right]
\right|
\le 2G.
\end{equation}

Therefore,
\begin{equation}
\begin{aligned}
\left|
S_{t,\beta}(q_t)
\right|
&=
\left|
\sum_{s=1}^{L_t}
\left[
g_{t,s,x_{t,s}}(\theta)
-
\mathbb{E}_{x_k\sim\pi_{\beta,\theta}}
\left[
g_{t,s,k}(\theta)
\mid
\mathcal{F}_{t,s-1}
\right]
\right]
\right|
\\
&\le
\sum_{s=1}^{L_t}
\left|
g_{t,s,x_{t,s}}(\theta)
-
\mathbb{E}_{x_k\sim\pi_{\beta,\theta}}
\left[
g_{t,s,k}(\theta)
\mid
\mathcal{F}_{t,s-1}
\right]
\right|
\\
&\le
\sum_{s=1}^{L_t} 2G
\\
&=
2L_tG
\\
&\le
2L_{\max}G.
\end{aligned}
\end{equation}

\emph{Bounding the derivative of the query-level score.}

For each token position $s$, define
\begin{equation}
\mu_{t,s}(\beta)
:=
\mathbb{E}_{x_k\sim\pi_{\beta,\theta}}
\left[
g_{t,s,k}(\theta)
\mid
\mathcal F_{t,s-1}
\right].
\end{equation}

Since $g_{t,s,k}(\theta)$ does not explicitly depend on $\beta$,
\begin{equation}
\begin{aligned}
\frac{\partial}{\partial\beta}\mu_{t,s}(\beta)
&=
\sum_{k\in\mathcal V}
g_{t,s,k}(\theta)
\frac{\partial}{\partial\beta}
\pi_{\beta,\theta}
\left(
x_k\mid\mathcal F_{t,s-1}
\right)
\\
&=
\sum_{k\in\mathcal V}
\pi_{\beta,\theta}
\left(
x_k\mid\mathcal F_{t,s-1}
\right)
g_{t,s,k}(\theta)
\left(
g_{t,s,k}(\theta)-\mu_{t,s}(\beta)
\right)
\\
&=
\mathbb{E}_{x_k\sim\pi_{\beta,\theta}}
\left[
g_{t,s,k}(\theta)^2
\mid
\mathcal F_{t,s-1}
\right]
-
\mu_{t,s}(\beta)^2
\\
&=
\operatorname{Var}_{x_k\sim\pi_{\beta,\theta}}
\left(
g_{t,s,k}(\theta)
\mid
\mathcal F_{t,s-1}
\right).
\end{aligned}
\end{equation}

Since $|g_{t,s,k}(\theta)|\le G$,
\begin{equation}
\operatorname{Var}_{x_k\sim\pi_{\beta,\theta}}
\left(
g_{t,s,k}(\theta)
\mid
\mathcal F_{t,s-1}
\right)
\le G^2.
\end{equation}

Therefore,
\begin{equation}
\frac{\partial}{\partial\beta}
S_{t,\beta}(q_t)
=
-
\sum_{s=1}^{L_t}
\operatorname{Var}_{x_k\sim\pi_{\beta,\theta}}
\left(
g_{t,s,k}(\theta)
\mid
\mathcal F_{t,s-1}
\right).
\end{equation}

Hence,
\begin{equation}
\left|
\frac{\partial}{\partial\beta}
S_{t,\beta}(q_t)
\right|
\le
L_tG^2
\le
L_{\max}G^2.
\end{equation}

\emph{Bounding the second-order variation of the downstream return.}
From Lemma~\ref{lemma2:First-Order Downstream Return Characterization},
\begin{equation}
\frac{\partial J_t(\beta)}{\partial\beta}
=
\mathbb E_{q_t\sim\pi_{\beta,\theta}}
\left[
Q_t(q_t)S_{t,\beta}(q_t)
\mid
\mathcal F_{t-1}
\right].
\end{equation}

Differentiating once more with respect to $\beta$ gives
\begin{equation}
\frac{\partial^2 J_t(\beta)}{\partial\beta^2}
=
\mathbb E_{q_t\sim\pi_{\beta,\theta}}
\left[
Q_t(q_t)
\left(
S_{t,\beta}(q_t)^2
+
\frac{\partial}{\partial\beta}
S_{t,\beta}(q_t)
\right)
\middle|
\mathcal F_{t-1}
\right].
\end{equation}

Using $|Q_t(q_t)|\le Q_{\max}$ together with the bounds above,

\begin{equation}
\begin{aligned}
\left|
\frac{\partial^2 J_t(\beta)}{\partial\beta^2}
\right|
&=
\left|
\mathbb{E}_{q_t\sim\pi_{\beta,\theta}}
\left[
Q_t(q_t)
\left(
S_{t,\beta}(q_t)^2
+
\frac{\partial}{\partial\beta}
S_{t,\beta}(q_t)
\right)
\right]
\right|
\\
&\le
\mathbb{E}_{q_t\sim\pi_{\beta,\theta}}
\left[
\left|Q_t(q_t)\right|
\left(
\left|S_{t,\beta}(q_t)\right|^2
+
\left|
\frac{\partial}{\partial\beta}
S_{t,\beta}(q_t)
\right|
\right)
\right]
\\
&\le
Q_{\max}
\left(
(2L_{\max}G)^2
+
L_{\max}G^2
\right)
\\
&=
Q_{\max}
\left(
4L_{\max}^2G^2
+
L_{\max}G^2
\right)
\\
&=
Q_{\max}G^2
\left(
4L_{\max}^2+L_{\max}
\right).
\end{aligned}
\end{equation}
Thus, defining
\begin{equation}
B_J
:=
Q_{\max}G^2
\left(
4L_{\max}^2+L_{\max}
\right),
\end{equation}
we obtain
\begin{equation}
\boxed{
\left|
\frac{\partial^2 J_t(\beta)}{\partial\beta^2}
\right|
\le B_J
}.
\end{equation}

This establishes the bounded second-order variation of the expected
downstream return and completes the proof.
\end{proof}

% \begin{theorem}[Performance Improvement under Value-Aligned Logit Intervention]
% \label{thm:value_aligned_improvement}

% Suppose the \textit{Bounded Logit Intervention} assumption
% (Assumption~\ref{assumption:bounded_intervention}) and the
% \textit{Value Alignment with Downstream Return} assumption
% (Assumption~\ref{assumption:value_alignment}) hold, and the conditions
% of Lemma~\ref{lemma3} are satisfied.

% Let
% \begin{equation}
% C_t
% :=
% \operatorname{Cov}_{q_t\sim
% \pi_H(\cdot\mid\mathcal F_{t-1})}
% \left(
% Q_t(q_t),
% S_t(\theta;q_t)
% \right).
% \end{equation}

% If the chosen intervention direction satisfies
% \begin{equation}
% \beta C_t>0,
% \end{equation}
% and
% \begin{equation}
% \left|
% J_t''(u)
% \right|
% \leq B_J
% \end{equation}
% for all $u$ between $0$ and $\beta$, then
% \begin{equation}
% J_t(\beta)-J_t(0)
% \geq
% \beta C_t
% -
% \frac{B_J}{2}\beta^2.
% \end{equation}

% Consequently, for any nonzero $\beta$ satisfying
% \begin{equation}
% 0<|\beta|<\frac{2|C_t|}{B_J},
% \end{equation}
% we have
% \begin{equation}
% \boxed{
% J_t(\beta)>J_t(0)
% }.
% \end{equation}
% \end{theorem}

\begin{proof}[Proof of Theorem~\ref{thm:local_improvement}]

Let
\begin{equation}
C_t
:=
\operatorname{Cov}_{q_t\sim
\pi_H(\cdot\mid\mathcal F_{t-1})}
\left(
Q_t(q_t),
S_t(\theta;q_t)
\right).
\end{equation}

By Lemma~\ref{lemma2:First-Order Downstream Return Characterization},
\begin{equation}
J_t'(0)=C_t.
\end{equation}

By Assumption~\ref{assumption:value_alignment}, the chosen intervention
direction satisfies
\begin{equation}
\beta C_t>0.
\end{equation}

We now apply a second-order Taylor expansion of $J_t(\beta)$ around
$\beta=0$. For sufficiently small nonzero $\beta$, Taylor's theorem
with Lagrange remainder gives
\begin{equation}
J_t(\beta)
=
J_t(0)
+
\beta J_t'(0)
+
\frac{\beta^2}{2}J_t''(\xi),
\end{equation}
for some $\xi$ lying between $0$ and $\beta$.

Therefore,
\begin{equation}
J_t(\beta)-J_t(0)
=
\beta C_t
+
\frac{\beta^2}{2}J_t''(\xi).
\end{equation}

By Lemma~\ref{lemma3},
\begin{equation}
\left|J_t''(\xi)\right|\leq B_J,
\end{equation}
and hence
\begin{equation}
J_t''(\xi)\geq -B_J.
\end{equation}

It follows that
\begin{equation}
\boxed{
J_t(\beta)-J_t(0)
\geq
\beta C_t
-
\frac{B_J}{2}\beta^2
}.
\end{equation}

Since $\beta C_t>0$, the right-hand side is strictly positive whenever
\begin{equation}
0<|\beta|<\frac{2|C_t|}{B_J}.
\end{equation}

Therefore,
\begin{equation}
\boxed{
J_t(\beta)>J_t(0)
}.
\end{equation}

Thus, when the intervention direction is locally aligned with downstream
value, a sufficiently small logit intervention improves the expected
downstream value relative to the frozen Host policy.
\end{proof}